\documentclass[10pt,twocolumn,letterpaper]{article}

\usepackage[pagenumbers]{cvpr} 

\usepackage{epsfig}
\usepackage{graphicx}
\usepackage{amsmath}
\usepackage{amsthm}
\usepackage{color}
\usepackage{booktabs}
\usepackage{multirow}
\usepackage{adjustbox}

\newtheorem{proposition}{Proposition}
\usepackage{algorithmic}
\usepackage{algorithm}
\usepackage{marvosym}
\usepackage{float}

\definecolor{cvprblue}{rgb}{0.21,0.49,0.74}
\usepackage[pagebackref,breaklinks,colorlinks,allcolors=cvprblue, urlcolor=red]{hyperref}

\definecolor{somegray}{rgb}{0.5, 0.5, 0.5}
\newcommand{\darkgrayed}[1]{\textcolor{somegray}{#1}}
\makeatletter
\newcommand*\titleheader[1]{\gdef\@titleheader{#1}}
\AtBeginDocument{%
	\let\st@red@title\@title
	\def\@title{%
		\vskip-3.5em
		\bgroup\normalfont\large\centering\@titleheader\par\egroup
		\vskip1.5em\st@red@title}
}
\makeatother

\titleheader{\darkgrayed{This paper has been accepted for publication at the \\
		IEEE/CVF Conference on Computer Vision and Pattern Recognition (CVPR), 2025.}}

\def\paperID{***} 
\def\confName{CVPR}
\def\confYear{2025}

\def\A{\mathbf{A}}
\def\B{\mathbf{B}}
\def\C{\mathbf{C}}
\def\d{\mathbf{d}}
\def\e{\mathbf{e}}
\def\f{\mathbf{f}}

\def\I{\mathbf{I}}

\def\L{\mathbf{L}}
\def\m{\mathbf{m}}
\def\M{\mathbf{M}}
\def\N{\mathbf{N}}
\def\n{\mathbf{n}}
\def\R{\mathbf{R}}
\def\u{\mathbf{u}}
\def\v{\mathbf{v}}
\def\x{\mathbf{x}}
\def\omg{\boldsymbol{\omega}}
\def\tht{\boldsymbol{\theta}}

\DeclareMathOperator*{\argmin}{arg\,min}
\DeclareMathOperator*{\rank}{rank}
\newcommand{\scenario}[1]{{\small \sf#1}\xspace}
\newcommand\blfootnote[1]{%
	\begingroup
	\renewcommand\thefootnote{}\footnote{#1}%
	\addtocounter{footnote}{-1}%
	\endgroup
}

\title{Full-DoF Egomotion Estimation for Event Cameras Using Geometric Solvers}

\author{
Ji Zhao\textsuperscript{1} \quad Banglei Guan\textsuperscript{2(\Letter)} \quad Zibin Liu\textsuperscript{2} \quad  Laurent Kneip\textsuperscript{3}\\
	{\small$^{1}$Independent Researcher, Beijing, China.} 
	{\small$^{2}$College of Aerospace Science and Engineering,} \\ {\small National University of Defense Technology, China.} 
	{\small$^{3}$Mobile Perception Lab, ShanghaiTech University, China.} \\
}

\begin{document}
\maketitle
\begin{abstract}
For event cameras, current sparse geometric solvers for egomotion estimation assume that the rotational displacements are known, such as those provided by an IMU. Thus, they can only recover the translational motion parameters. Recovering full-DoF motion parameters using a sparse geometric solver is a more challenging task, and has not yet been investigated. In this paper, we propose several solvers to estimate both rotational and translational velocities within a unified framework. Our method leverages event manifolds induced by line segments. The problem formulations are based on either an incidence relation for lines or a novel coplanarity relation for normal vectors. We demonstrate the possibility of recovering full-DoF egomotion parameters for both angular and linear velocities without requiring  extra sensor measurements or motion priors. To achieve efficient optimization, we exploit the Adam framework with a first-order approximation of rotations for quick initialization. Experiments on both synthetic and real-world data demonstrate the effectiveness of our method. The code is available at \url{https://github.com/jizhaox/relpose-event}.
\end{abstract}

\blfootnote{\Letter: Corresponding author. \\ Email: \texttt{zhaoji84@gmail.com}, \ \texttt{guanbanglei12@nudt.edu.cn}, \texttt{liuzibin19@nudt.edu.cn}, \ \texttt{lkneip@shanghaitech.edu.cn}}
\section{Introduction}
\label{sec:intro}

In recent years, neuromorphic event-based cameras have drawn increasing attention~\cite{gallego2022event}. They capture pixel-wise intensity changes and generate  asynchronous event streams indicating relative brightness changes, and thereby create favorable bandwidth latency trade-offs.
The present paper looks at the fundamental task of event-based motion and structure estimation.

Compared against traditional frame-based cameras, event cameras capture a different encoding of visual information. As a result, new algorithms need to be developed to process event streams. 
The recovery of egomotion for event cameras has for example been achieved by aligning warped events~\cite{gallego2017accurate,gallego2018unifying,gallego2019focus}, exploiting normal flow~\cite{shiba2024secrets,lu2024event,ren2024motion}, coupling with frame-based cameras or an inertial measurement unit (IMU)~\cite{zhu2017event,mueggler2018continuous}, and by using learning-based methods~\cite{zhu2019unsupervised,gehrig2020event,ye2020unsupervised,paredes-valles2021back,klenk2023deep,pellerito2024deep}.

Of particular interest to the present work are sparse geometric solvers for motion estimation from samples of raw events. The most relevant work to ours is given by Gao \etal \cite{gao2023fivepoint,gao2024npoint}, who introduce a manifold for line events, called \emph{eventail}.
Based on the theoretical foundation of eventails and the assumption of locally constant velocity, temporal asynchronous events can be used for linear velocity estimation. 
Gao \etal develop a minimal solver to recover the linear velocity of the event camera and line structures. By eliminating the scale-ambiguity of translation and scene structure, the problem has $5$ unknown DoFs, thus requiring a minimum of $5$ events.
Later, an efficient $N$-point linear solver was developed~\cite{gao2024npoint}. 
However, in the aforementioned methods, angular velocities are measured by an IMU. The direct estimation of rotational parameters from events remains yet to be explored.

Despite their success, previous mainstream methods estimate either rotational parameters, as seen in \cite{gallego2017accurate, gallego2018unifying,gallego2019focus}, or translational parameters, as in \cite{gao2023fivepoint,gao2024npoint,lu2024event}.
By contrast, this paper estimates egomotion containing both rotational and translational parameters. 
To achieve this goal, we propose a full-DoF egomotion estimation method using the geometry of eventail manifolds~\cite{gao2024npoint}. 
When the unknown angular velocity is involved, the minimum number of events required increases to $8$, making the problem more complex compared to the case where rotational displacement is known.
Besides this direct extension, we also develop a novel geometry for eventail manifolds relying on the normal flow and yielding good overall performance.

Though this paper focuses on egomotion estimation using event streams, the proposed geometry and formulations can be understood as a general framework for infinitesimal motion estimation. Whenever a visual sensor is moving under a locally constant velocity and samples (either asynchronously or synchronously) from image gradients or normal flow of edgelets are generated, the present framework can be used to recover angular and linear velocities. This framework naturally supports event cameras due to its asynchronous and differential working principle. 
When applied to frame-based cameras, the frame rate must be sufficiently high to accurately capture egomotion, due to the synchronous nature of sampling.
The contributions of this paper are summarized as follows.
\begin{itemize}
	\item {\bf Geometry}: We propose a novel coplanarity relation for the eventail manifold. This foundational geometric theory differs from the incidence relation proposed in previous methods \cite{gao2023fivepoint,gao2024npoint}, and promises a more reliable egomotion estimation solution with strong potential in a wide range of applications.
	\item {\bf Formulations of Full-DoF Motion Estimation}: We develop several sparse geometric solvers that solve for full-DoF egomotion parameters. They are the first of its kind.
	The problem formulations are based on either an incidence relation for lines or a novel coplanarity relation for plane normals.
	All resulting optimization problems aim to estimate both the rotational and translational parameters, i.e., the angular and linear velocities.
	\item {\bf Optimization}: We have employed the Adam optimization framework~\cite{kingma2015adam} to address the resulting optimization problems. To further enhance efficiency, we have employed a first-order approximation for rotations alongside a \emph{data compression} technique~\cite{helmke2007essential,kneip2013direct}.
	\item {\bf Theoretical Analysis}: We have conducted a theoretical analysis of pure rotation cases, an area that has not been explored in geometric solvers. Our method not only identifies but also handles pure rotation scenarios effectively
\end{itemize}

\section{Related Work}
\label{sec:relwork}

Most existing works simplify the egomotion estimation problem by assuming pure rotational motion and constant angular velocity over short time intervals~\cite{gallego2017accurate,gallego2018unifying,gallego2019focus,liu2020globally,peng2021globally,gu2021spatio,liu2021spatiotemporal,guo2024cmax}.
Moreover, some works recover linear velocity only~\cite{ieng2017event,lu2024event,gao2023fivepoint,gao2024npoint} under the assumption of known angular velocities.
A few works recover both angular and linear velocities events with depth~\cite{huang2023progressive,ren2024motion} or stereo event cameras~\cite{zhu2019unsupervised}. 
An angular velocity estimator for planar ground vehicles was proposed in~\cite{xu2024event}.
A continuous-time framework for event-inertial fusion is presented in~\cite{mueggler2018continuous}.  
In \cite{ren2024motion}, sparse geometric solvers are used to obtain an initial egomotion estimation, followed by continuous-time refinement.

Contrast maximization (CMax) \cite{gallego2017accurate,gallego2018unifying,gallego2019focus,shiba2024secrets} provides a unified framework to solve a range of problems with event cameras, including egomotion, depth and optical flow estimation. The CMax framework has been improved in terms of efficiency and robustness~\cite{gu2021spatio,liu2021spatiotemporal,huang2023progressive,hamann2024motion}. 
In \cite{liu2020globally,peng2021globally}, globally optimal solutions have been proposed for the CMax framework.
In \cite{seok2020robust}, the constant velocity assumption was dropped and B{\'e}zier curve was estimated for linear interpolation.
The dispersion minimization (DMin) framework \cite{nunes2022robust}  minimizes dispersion of features associated with raw events directly in feature space.

Line extraction methods from event streams have been proposed~\cite{brandli2016elised,everding2018low}. Additionally, the observation of line events has been used to estimate camera motion via visual geometry~\cite{kim2014simultaneous,legentil2020idol}.
In \cite{peng2021continuous}, a closed-form linear velocity initialization method was developed based on the trifocal tensor geometry. 
In \cite{gao2023fivepoint,gao2024npoint}, solvers are proposed for estimating linear velocity and line parameters. 
These solvers do not estimate angular velocity and assume that rotations are provided by IMU measurements.

It is well known that scene depth and egomotion determine the optical flow~\cite{longuet1980interpretation}. Thus, optical flow provides constraints for scene depth and egomotion~\cite{barranco2021joint}. The normal flow can be obtained based on fitting local spatio-temporal planes on the time surface~\cite{benosman2013event,lagorce17hots,hordijk2018vertical}.
For event cameras, normal flow has been used to recover the linear velocity~\cite{lu2024event,ren2024motion}.

To improve robustness and scale recovery, it is good practice to equip an event camera with other kinds of sensors and exploit sensor fusion. It is common to consider an additional IMU, and many frameworks for event-based visual-inertial odometry have been developed~\cite{kim2014simultaneous,reinbacher2017,zhu2017event,vidal2018ultimate,mueggler2018continuous,xu2023tight,guo2024cmax,pellerito2024deep}. A line-based event-inertial odometry framework was proposed in~\cite{legentil2020idol}, which integrates both event data and inertial measurements for enhanced performance. A continuous-time framework for event-inertial fusion is presented in~\cite{mueggler2018continuous}. CMax and rotation-only bundle adjustment have been used to develop a SLAM system~\cite{guo2024cmax}. 
 
There are several learning-based methods for egomotion regression and event-based visual odometry. Early attempts focused on unsupervised learning~\cite{zhu2019unsupervised,ye2020unsupervised,paredes-valles2021back} and spiking neural network~\cite{gehrig2020event}.
Recently, several methods have been developed to enhance the applicability of learning-based methods, such as generalizing them to out-of-distribution scenarios~\cite{klenk2023deep} and fusing event data into frames end-to-end~\cite{pellerito2024deep}.

For frame-based cameras, relative pose can be recovered by decoupling of rotation from translation \cite{kneip2013direct,chng2020monocular,muhle2022probabilistic,guan2024six,xu2024accurate}. The rotation is determined through eigenvalue minimization of a matrix that is constructed based on the observations~\cite{kneip2013direct}. The rotation-only bundle adjustment method~\cite{lee2021rotation} extends two-view relative pose estimation to pose graph optimization. In this paper, we apply these  rotation-translation decoupling and eigenvalue minimization techniques to event cameras.

\section{Egomotion Estimation Based on Line Incidence}
\label{sec:method}

In this section, we propose a method for recovering both rotational and translational parameters. In addition, we discuss the optimality conditions and how to handle pure rotation cases. 

Let there be a set of $M$ static lines $\{\L_i\}_{i=1}^M$ in 3D space. The Pl{\"u}cker coordinates of a line are denoted as $\L = [\d^\intercal, \m^\intercal]^\intercal$, where $\d \in \mathbb{R}^3$ represents the direction and $\m \in \mathbb{R}^3$ the moment.
A moving event camera observes these lines and generates $N_i$ events $\mathcal{E}_i = \{e_{ij}\}_{j=1}^{N_i}$ for the $i$-th line.
Each event $e_{ij} = (\x_{ij}, \tau_{ij}, p_{ij})$ consists of its normalized image coordinate $\x_{ij}$, timestamp $\tau_{ij}$, and polarity $p_{ij} \in \{+1, -1\}$. 
In this paper, we do not use the polarity information.

Due to the high-rate imaging nature of event cameras, a short time interval can be chosen such that the camera motion is well-approximated by constant angular and linear velocities.
The event camera's angular velocity is $\omg = [\omega_x, \omega_y, \omega_z]$, and its linear velocity is $\v = [v_x, v_y, v_z]$.

\subsection{Preliminaries of Incidence Relation}
\begin{figure}[tbp]
	\centering
	\includegraphics[width=0.95\linewidth]{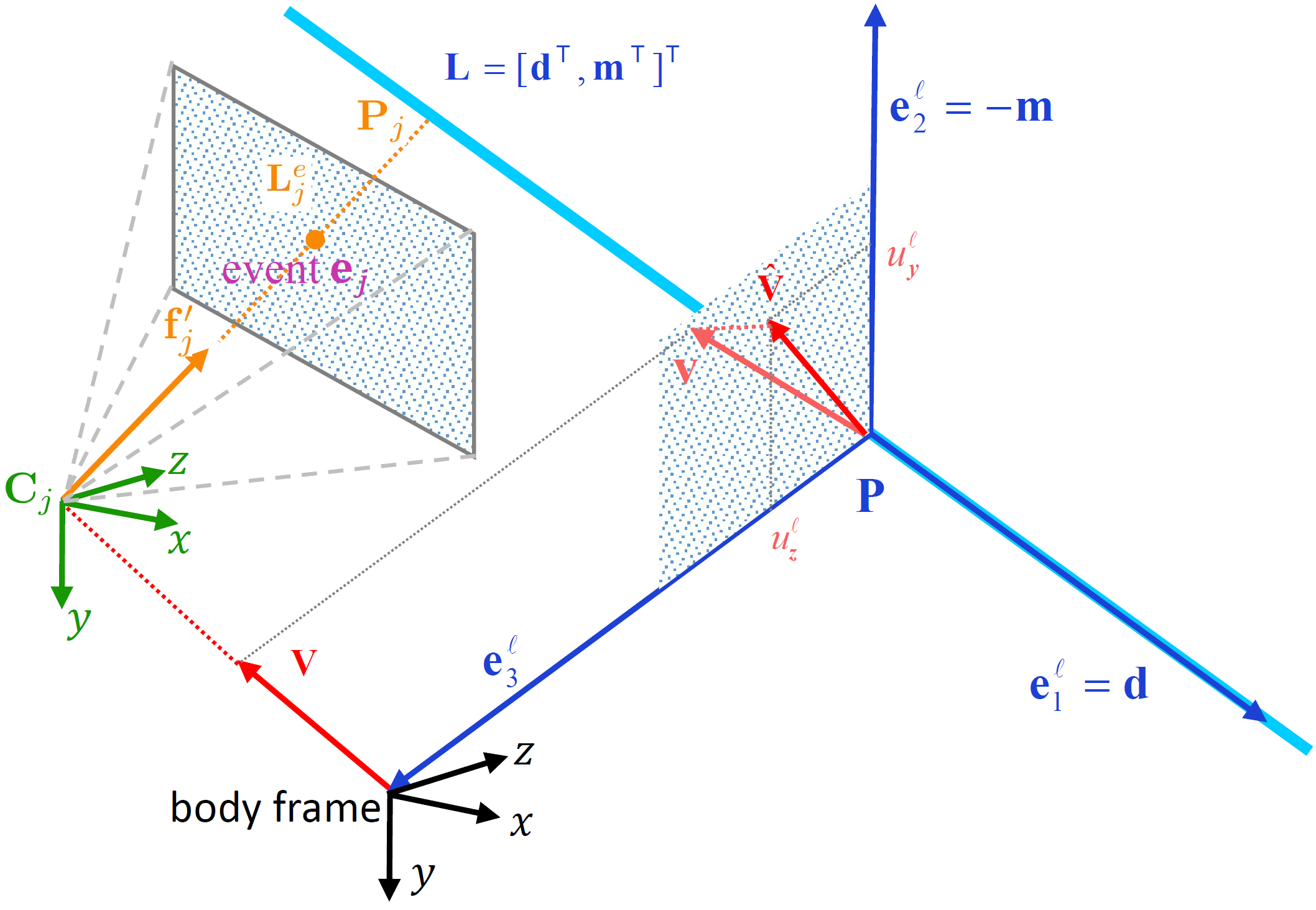}
	\caption{Incidence relation between the observed line $\L$ and line $\L^e_j$ of the $j$-th event. The line $\L^e_j$ is consistent with the bearing vector $\f'_j$. The vector $\hat{\mathbf{v}}$ represents the projection of the translation $\mathbf{v}$ onto the plane spanned by the vectors $\mathbf{e}_2^\ell$ and $\mathbf{e}_3^\ell$, which is filled by dot patterns. Due to the aperture problem, only $u^\ell_y$ and $u^\ell_z$ components are observable.}
	\label{fig:frame}
\end{figure}

Under the constant velocity assumption, the events of a line during a time interval form an eventail manifold~\cite{gao2023fivepoint,gao2024npoint}.
We briefly outline the geometry of the eventail manifold. 
The detailed derivation can be found in the original papers.
The geometry is essentially defined by an incidence relation. 
For simplicity, we first consider the case of one line $\L$, and thus drop the index $i$ from the observed lines.
Let line $\L^e_j$ be consistent with the bearing vector of an individual event $e_j$ triggered at time $\tau_j$.
As shown in \cref{fig:frame}, the line $\L^e_j$ (\textcolor[RGB]{249, 137, 4}{orange line}) should intersect the observed line $\L$ (\textcolor[RGB]{0,207,255}{light blue line}).

Let the time interval of the manifold be $[\tau_s - \Delta \tau, \tau_s + \Delta 
\tau]$, centered at time $\tau_s$. The offset of timestamp $\tau_j$ relative to the interval midpoint is $t_j = \tau_j - 
\tau_s$. 
In brief, $\tau$ and $t$ correspond to absolute and relative timestamps, respectively. Hereafter, we will mainly use the relative timestamps.
Let the camera frame at time $\tau_s$ be the \emph{body frame}.
Under the constant velocity assumption, the orientation $\R_j$ and origin $\C_j$ of the camera frame at time $t_j$ are determined by
\begin{align}
	\R_j &= \R(\omg; t_j) = \exp([t_j \omg]_\times),
	\label{eq:orientation} \\
	\C_j &= t_j \v, \label{eq:origin}
\end{align}
where $\exp(\cdot)$ denotes the exponential map, and $[\cdot]_\times$ is the skew-symmetric matrix of a $3$-dimensional vector.
The orientation is computed by the Rodrigues' rotation formula~\cite{ma2004invitation}. 
The motion described is a \emph{screw motion} with a fixed instantaneous screw axis~\cite{lynch2017modern}.

Let $\f_j$ be the bearing vector for event $e_j$ in the camera reference at time $t_j$. $\f_j$ is obtained by homogenizing the image coordinate $\x_j$ and then normalizing it.
When expressed in the body frame, the bearing vector becomes 
\begin{align}
\f'_j = \R_j \f_j.
\label{eq:bearing-tran}
\end{align}

As shown in \cref{fig:frame}, a \emph{line-dependent frame} $\R^\ell = [\e_1^\ell, \e_2^\ell, \e_3^\ell] \in \text{SO}(3)$ can be uniquely constructed for each line $\L$. 
We select a point $\mathbf{P}$ closest to the origin, such that it is perpendicular to $\d$. Since there is a scale-ambiguity, we may fix the distance from $\mathbf{P}$ to the origin to be unity.
The basis of the frame is $\e_1^\ell = \d$, $\e_3^\ell = -\mathbf{P}$, and $\e_2^\ell = \e_3^\ell \times \e_1^\ell$.
The line parameter $\tht^\ell$ can be extracted using the matrix logarithm $\tht^\ell = \left( \log(\R^\ell) \right)^{\vee} = [\theta^\ell_x, \theta^\ell_y, \theta^\ell_z]$. $(\cdot)^{\vee}$ is the vee operator that maps a skew-symmetric matrix to its corresponding vector representation.

Finally, the linear velocity $\v$ of the camera can be expressed in frame $\R^\ell$ by
\begin{align}
	\v = \R^\ell \u^\ell,
	\label{eq:velo-trans}
\end{align}
where $\u^\ell = [u^\ell_x, u^\ell_y, u^\ell_z]$ is the linear velocity in the line-dependent frame. Due to the aperture problem, only $u^\ell_y$ and $u^\ell_z$ components are observable. In other words, $u^\ell_x$ vanishes in \cref{eq:velo-trans}.
By utilizing the frame $\R^\ell$, the incidence relation between $\L^e_j$ and line $\L$ yields the following equation
\begin{align}
	t_j {\f'_j}^\intercal (u^\ell_z \e^\ell_2 - u^\ell_y \e^\ell_3) + {\f'_j}^\intercal \e^\ell_2 = 0.
	\label{eq:incidence}
\end{align}

\subsection{Rotation Estimation from A Single Line}
The incidence relation of \cref{eq:incidence} comprises eight unknowns: three associated with the angular velocity $\omg = [\omega_x, \omega_y, \omega_z]$, two with partial linear velocity $\u^\ell = [0, u^\ell_y, u^\ell_z]$, and three with the line parameters $\tht^\ell$. 
If rotational parameters are known, the bearing vectors $\{\f'_j\}_j$ can be pre-rotated using known rotations, leaving only five unknowns in the resulting problem. However, if the rotational parameters are unknown, they can be treated as unknowns, resulting in a more complex problem.
 
By arranging the incidence equations of $N$ points ($N \ge 8$) in \cref{eq:incidence}, we can express them in matrix form
\begin{align} 
	\underbrace{
	\begin{bmatrix}
		t_1 {\f'_1}^\intercal & {\f'_1}^\intercal \\
		\vdots & \vdots \\
		t_N \f'^\intercal_N & \f'^\intercal_N
	\end{bmatrix}
	}_{\dot{=} \A(\omg) \in \mathbb{R}^{N\times 6}}
	\underbrace{
	\begin{bmatrix}
		u^\ell_z \e^\ell_2 - u^\ell_y \e^\ell_3 \\
		\e^\ell_2
	\end{bmatrix}
	}_{\dot{=}\x \in \mathbb{R}^{6\times 1}}
	= \mathbf{0}.
	\label{eq:incidence-matrix}
\end{align}
Note that $\f'_j$ are the unrotated versions of the $\f_j$ in \cref{eq:bearing-tran}.
Thus, the matrix $\A(\omg)$ depends on the unknown angular velocity. This formulation decouples rotation estimation from translation and line structure estimation. One can easily verify that the minimal configuration is $N = 8$. By taking the sub-determinants of matrix $\A$ and setting them to $0$, we find that at least $8$ rows are required to constrain the 3~DoF of the rotation, given that there are $6$ columns.

Since $\A(\omg)$ has a non-empty null space, we seek the angular velocity $\omg$ that minimizes the singular value of $\A$. Let $\M(\omg)$ be defined as
\begin{align}
	\M(\omg) \dot{=} \A^\intercal(\omg) \A(\omg).
	\label{eq:M-matrix}
\end{align}
$\M$ is a real symmetric $6\times 6$ matrix, and it is positive semidefinite (PSD).
Then the objective can be equivalently written as minimizing the smallest eigenvalue $\lambda_{\text{min}}$ of $\M$:
\begin{align}
	\omg^* = \argmin_{\omg} \ \lambda_{\text{min}}(\M(\omg)).
	\tag{\textcolor{red}{IncMin}}
	\label{eq:opt1}
\end{align}

When $N=8$, this problem is exactly-determined. When $N > 8$, the problem becomes over-determined. 
In both cases, the first order optimality conditions or the Karush-Kuhn-Tucker (KKT) conditions have exactly the same expression. Although the equation systems for KKT conditions are intractable, they unify the two cases into a minimal problem, regardless of how many events are used.

\subsection{Rotation Estimation from A Batch of Lines}

To make the angular velocity recovery more accurate and stable, we can also use a batch of $M$ lines ($M \ge 2$).
Define 
\begin{align}
	\M_i(\omg) \dot{=} \A_i^\intercal(\omg) \A_i(\omg),
	\label{eq:M-matrix-i}
\end{align}
where $\A_i$ is associated with the $i$-th line. Then the angular velocity recovery can be obtained by optimizing
\begin{align}
	\omg^* = \argmin_{\omg} \ \sum_{i=1}^M \lambda_{\text{min}}(\M_i(\omg)).
	\tag{\textcolor{red}{IncBat}}
	\label{eq:opt1-i}
\end{align}

\subsection{Translation Estimation}
Once the angular velocity has been recovered, whether it is from a single line or a batch of lines, we can apply a linear solver from either \cite{gao2024npoint} or \cref{sec:npt-cop} of this paper to recover the linear velocity and line structure.
For the linear solver in \cite{gao2024npoint}, we prove that it cannot handle pure rotation cases. We propose a method to identify and effectively handle pure rotation scenarios. More details are provided in the supplementary material.

\section{Egomotion Estimation Based on Coplanarity of Normal Vectors}

In this section, we propose a coplanarity relation among event-induced plane normals. This coplanarity geometry, which is based on normal flow, enables a novel formulation for egomotion estimation.

\begin{figure}[tbp]
	\centering
	\includegraphics[width=0.75\linewidth]{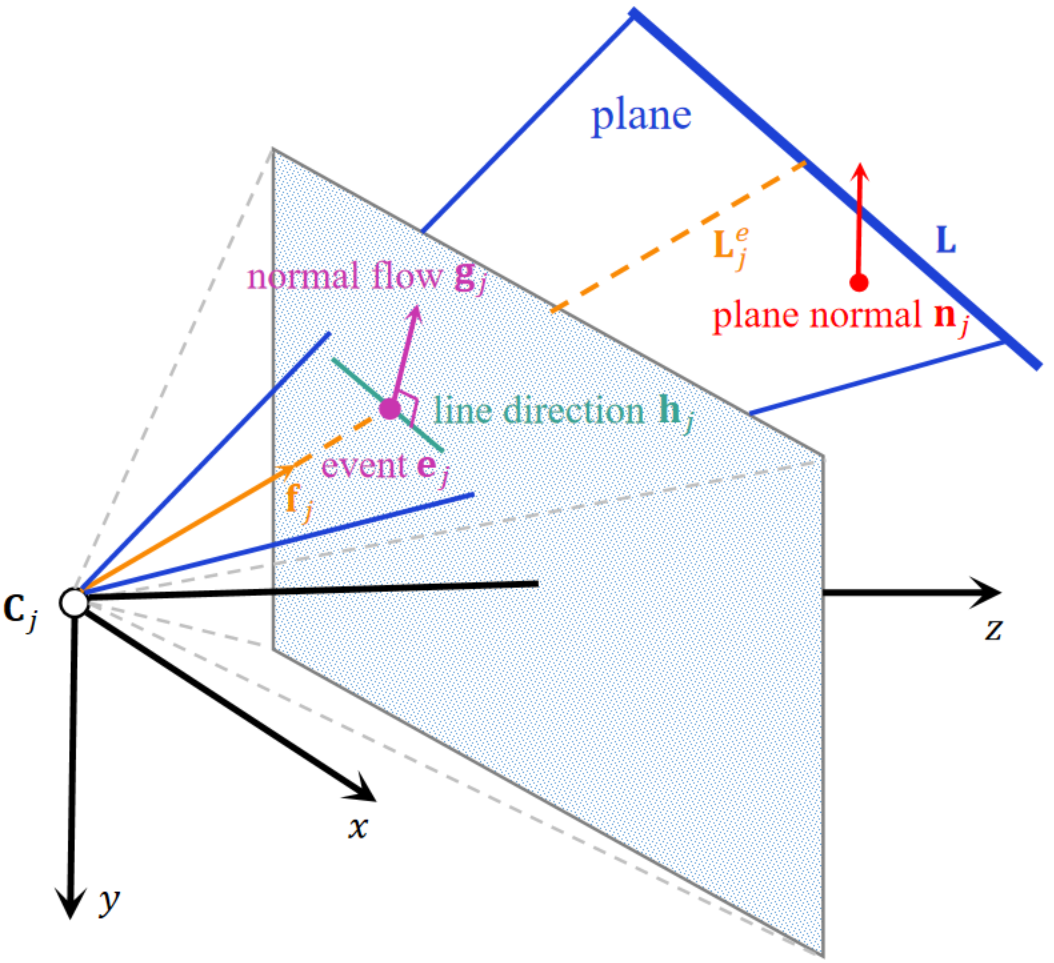}
	\caption{Coplanarity relation between plane normal vectors. Plane normal $\n_j$ can be computed from the event $e_j$ and its normal flow $\mathbf{g}_j$. The line direction vector $\mathbf{h}_j$ in the image plane is perpendicular with $\mathbf{g}_j$ within the image plane. The line $\L$ is orthogonal to the plane normal set $\{\n'_j\}_{j=1}^N$.}
	\label{fig:flow}
\end{figure}

\subsection{Coplanarity Relation}
Let us begin with a single line for simplicity of notation.
Consider $\L = [\d^\intercal, \m^\intercal]^\intercal$ as a static line in 3D space. 
A moving event camera observes this line and generates $N$ events $\mathcal{E} = \{e_{j}\}_{j=1}^{N}$.

For each event $e_j$, let $\n_{j}$ be the normal vector of the plane crossing $\L$ and $\L^e_j$ as illustrated in \cref{fig:flow}.
Let $\n'_j$ be the normal vector $\n_j$ expressed in the body frame.
Since the 3D line $\L$ lies in all the planes induced by the event set $\mathcal{E}$, the line direction $\d$ is perpendicular to all normals $\{\n_i\}_{i=1}^N$. This coplanarity relation gives rise to
\begin{align}
	\underbrace{[\n'_1, \cdots, \n'_N]^\intercal}_{\dot{=} \B \in \mathbb{R}^{N\times 3}} \d = \mathbf{0},
	\label{eq:coplanarity}
\end{align}
It can be observed that matrix $\B$ is rank-deficient.
This property forms the basis of our developed method. 
The coplanarity relation exhibits a duality with the incidence relation. For example, normal vector $\n_j$ is parallel to the moment $\m_j$ of line $\L^e_j$, and is also parallel to axis $\e_2^\ell$ of the line-dependent frame.

As shown in \cref{fig:flow}, plane normals can be derived from events and their normal flow. The procedure is summarized as follows. (1)~We compute the normal flow $\{[g^x_{j}, g^y_{j}]\}_{j=1}^N$ for the events using methods based on time surface maps~\cite{benosman2013event,lagorce17hots}.
(2)~Since normal flow $\mathbf{g}_j = [g^x_j, g^y_j]$ of event $e_j$ is parallel to the brightness gradient in the image plane, the line direction $\mathbf{h}_j = [h^x_j, h^y_j]$ in the image plane must be perpendicular to the normal flow, i.e., $[h^x_j, h^y_j] \propto [-g^y_j, g^x_j]$. 
(3)~Given event $e_j$ and its associated line direction $\mathbf{h}_j$ in the image plane, the plane normal $\n_j$ is determined by $\n_j \propto \f_j \times [h^x_j, h^y_j, 0]^\intercal \propto \f_j \times [-g^y_j, g^x_j, 0]^\intercal$. 
(4)~The plane normal $\n'_j$ expressed in the body frame is 
\begin{align}
	\n'_j = \R_j \n_j.
	\label{eq:nml-tran}
\end{align}

Based on the constructed plane normals, we will discuss how to recover angular and linear velocities. For simplicity in the following derivations, we introduce the matrix 
\begin{align}
	\N \dot{=} \B^\intercal \B = \sum_{j=1}^N \n'_j {\n'_j}^\intercal,
	\label{eq:matrix-N}
\end{align}
where $\N$ is a real symmetric $3\times 3$ matrix, and it is positive semidefinite.

\subsection{A Linear Solver for Translation Estimation}
\label{sec:npt-cop}
Based on the coplanarity relation, we develop an $N$-point linear solver to recover the translation and line structure. 
The linear solver is applied once the rotational parameters are available. The rotational parameters can be obtained by the angular velocity recovery method described in this paper, or provided directly by a sensor.

According to \cref{eq:coplanarity}, the line direction vector $\d$ is the eigenvector corresponding to the smallest eigenvalue of $\N$ (or the vector corresponding to the smallest singular value of $\B$). 
As shown in \cref{fig:frame}, line $\L^e_j$ coincides with the bearing vector $\f'_j$, and the camera center $t_j \v$ lies on this line. So the moment of line $\L^e_j$ is determined by $\m_j = (t_j \v) \times \f'_j$. The bearing vector $\f'_j$ can be pre-rotated using the known rotations.  Since line $\L^e_j = [{\f'_j}^\intercal, \m_j^\intercal]^\intercal$ intersects line $\L = [\d^\intercal, \m^\intercal]^\intercal$, we have
\begin{align}
	& \d^\intercal ((t_j \v) \times \f'_j) + {\f'_j}^\intercal \m = 0 \\
	\Rightarrow &
	\begin{bmatrix}
		t_j(\f'_j \times \d)^\intercal & {\f'_j}^\intercal
	\end{bmatrix}
	\begin{bmatrix}
		\v \\ \m
	\end{bmatrix}
	=0 \label{eq:linear-velo} \\
	\Rightarrow &
	\begin{bmatrix}
		t_j(\f'_j \times \d)^\intercal & {\f'_j}^\intercal
	\end{bmatrix}
	\begin{bmatrix}
		\v + \alpha \d \\ \m
	\end{bmatrix}
	=0, \ \forall \alpha \in \mathbb{R}.
	\label{eq:linear-velo-factor}
\end{align}
In above derivation, the last equation holds because $t_j(\f'_j \times \d)^\intercal (\alpha \d) = 0$. Thus, the unknown factor $\alpha$ cannot be determined given a single line. This is known as the aperture problem, i.e., the component of velocity $\v$ parallel to line direction $\d$ is unobservable when using a single line.

\cref{eq:linear-velo-factor} results in a linear $N$-point solver. By applying singular value decomposition (SVD) to the stacked matrix of $N$ points, the linear velocity $\v$, which includes an unknown component $\alpha \d$, and the line moment $\m$ can be obtained. Due to the scale-ambiguity in $\v$ and $\m$, at least $5$ events are needed to recover the unknowns. 
To remove the unknown component $\alpha \d$ from $\v$, non-parallel lines from multiple events should be used. 
Finally, the full linear velocity $\v$ can be obtained by the \emph{velocity averaging} method developed in \cite{gao2024npoint}. It is known that the linear velocity $\v$ can only be recovered up to an unknown scale, but its sign can be correctly determined, as discussed in \cite{gao2024npoint}.

\subsection{Rotation Estimation}

Given plane normal set $\{\n'_j\}_j$ by \cref{eq:nml-tran}, we have 
\begin{align}
	\N(\omg) = \sum_{j=1}^N (\R(\omg; t_j) \n_j) (\R(\omg; t_j) \n_j)^\intercal.
	\label{eq:N-matrix-nf}
\end{align}
It can be seen that each element of matrix $\N$ depends only on the angular velocity $\omg$. 

Considering $\B$ is rank deficient in \cref{eq:coplanarity}, we seek the angular velocity $\omg$ that minimizes the smallest singular value of $\B$.
This problem can be formulated as minimizing the smallest eigenvalue $\lambda_{\text{min}}$ of $\N(\omg)$ by
\begin{align}
	\omg^* = \argmin_{\omg} \  \lambda_{\text{min}}(\N(\omg))
	\tag{\textcolor{red}{CopMin}}.
	\label{eq:opt_lmd_nf}
\end{align}

If observations from a batch of $M$
lines are given, the optimization problem becomes
\begin{align}
	\omg^* = \argmin_{\omg} \ \sum_{i=1}^M \lambda_{\text{min}}(\N_i(\omg))
	\tag{\textcolor{red}{CopBat}},
	\label{eq:opt-lmd_nf_i}
\end{align}
where $\N_i$ is the matrix corresponding to the $i$-th line.

\section{Optimization}
In this section, we discuss how to solve these problems effectively and efficiently. 

\subsection{Rotation Parametrizations}
According to Rodrigues' rotation formula, the rotation described in \cref{eq:orientation} can be explicitly written as
\begin{align}
	\R_j &= \exp([t_j \omg]_\times) \nonumber \\
	&= \I + \frac{\sin(\theta_j)}{\theta_j} [t_j \omg]_\times + \frac{1-\cos(\theta_j)}{\theta_j^2} [t_j \omg]^2_\times,
	\label{eq:rot-exact}
\end{align}
where $\theta_j = \|t_j \omg\|$ is the rotation angle, and $t_j \omg/\theta_j$ is the rotation axis.

To simplify $\R_j$, note that $\sin(\theta) \approx \theta$ and $\cos(\theta) \approx 1 - \frac{\theta^2}{2}$ when angle $\theta$ is small.
By preserving only the linear term in $\omg$, we have a first-order approximation of the rotation as
\begin{align}
	\R_j &\approx \I + [t_j \omg]_\times \\
	&= 
	\begin{bmatrix}
		1 & -t_j \omega_z & t_j \omega_y \\
		t_j \omega_z & 1 & -t_j \omega_x \\
		-t_j \omega_y & t_j \omega_x & 1
	\end{bmatrix},
	\label{eq:rot-approx}
\end{align}
where $\I_{3\times 3}$ is an identity matrix. 
When rotation angle is small, the first-order approximation of rotations has been applied in geometric vision tasks to simplify related equations~\cite{stewenius2007efficient,ventura2014approximated,ventura2015efficient,gallego2017accurate}. 
This first-order approximation is reasonable in our problems since the event set spans a very short time interval, which gives a very small angle when multiplied by an angular velocity. 

\subsection{Adam Optimizer}

The Adam (adaptive moment estimation) optimizer~\cite{kingma2015adam} is widely used in deep learning. It is based on adaptive estimation of first-order and second-order moments. 
We found it has better performance than the Gauss-Newton and Levenberg-Marquardt optimizers for our problems.
Thus we use it to solve the optimization problems. 

{\bf Exact Rotation.}
In order to apply the Adam optimizer, the gradient of the objective function must be computed. 
Here we summarize the calculation of the objectives.

For the incidence formulation, the objective requires computing the smallest eigenvalue of $6\times 6$ matrices, for which no closed-form solution exists. Thus, the gradient also has no closed-form expression. Instead, we approximate the gradient numerically using the first-order finite difference method (FDM).
For the coplanarity formulation, the objective requires computing the smallest eigenvalue of 
$3\times 3$ matrices. In this case, closed-form formulas for the gradient are available, and we use these formulas to compute the gradients directly.

{\bf First-Order Approximated Rotation.}
The first-order approximated rotation allows for the application of the \emph{data compression} technique~\cite{helmke2007essential,kneip2013direct} to accelerate the computation of objectives and gradients in optimization. This is because certain intermediate results can be pre-computed and reused throughout the optimization process. 
When using the first-order approximated rotation, the Adam optimizers are similar to their aforementioned counterparts with exact rotation. The only difference lies in the computation of the objectives and their gradients. 

Although the approximated rotation parameterization introduces some approximation errors, it facilitates an efficient data compression technique. This makes each iteration of the optimizer having linear computational complexity with respect to the number of lines. By contrast, each iteration with exact rotation has linear computational complexity with respect to the total number of events.

{\bf Cascade of Two Rotation Parametrizations.}
To leverage the advantages of both rotation parameterizations, we propose a cascade method. First, by applying the approximated rotation parameterization, we efficiently solve the problem using the data compression technique. This solution serves as the initialization for solving the problem with the exact rotation parameterization.
In the subsequent experiments, we test both the exact and first-order approximated rotations, as well as their cascade.

\section{Experiments}
\label{sec:exp}

\begin{table*}[t]
	\centering
	\begin{tabular}{|l|l|r|r|r|r|r|r|}
		\hline
		formulation & rotation  & median $\varepsilon_{\text{ang}}$ & median $\varepsilon_{\text{lin}}$ ($^\circ$) & SR1(\%) & SR2(\%) & median objective & runtime \\ \hline
		\multirow{3}{*}{\scenario{IncBat}}
		& \scenario{+approx} & $1.4\times 10^{-2}$ & $1.6\times 10^{-1}$ & $32.4$ & $94.7$ & $1.2\times 10^{-8}$ & $16.8$ ms\\ 
		& \scenario{+exact} & $3.8\times 10^{-4}$ & $4.2\times 10^{-3}$ & $98.7$ & $98.9$ & $6.7\times 10^{-12}$ & $48.7$ ms\\ 
		& \scenario{+cascad} & $1.6\times 10^{-4}$ & $1.8\times 10^{-3}$ & $98.9$ & $99.2$ & $3.1\times 10^{-12}$ & $48.6$ ms\\ 
		\hline
		\multirow{3}{*}{\scenario{CopBat}} 
		& \scenario{+approx} & $1.9\times 10^{-2}$ & $2.1\times 10^{-1}$ & $21.9$ & $85.8$ & $1.0\times 10^{-7}$ & $16.7$ ms\\ 
		& \scenario{+exact} & $2.8\times 10^{-4}$ & $3.3\times 10^{-3}$ & $94.9$ & $96.7$ & $3.8\times 10^{-11}$ & $32.7$ ms \\ 
		& \scenario{+cascad} & $9.3\times 10^{-5}$ & $1.1\times 10^{-3}$ & $96.6$ & $97.3$ & $2.6\times 10^{-11}$ & $17.1$ ms\\ 
		\hline
	\end{tabular}
	\caption{Runtime and numerical stability for noise-free synthetic data. The configuration is $M=5$ and $N=100$. SR1 and SR2 represent success rate (SR) with thresholds of $0.01$ and $0.05$, respectively.}
	\label{tab:rslt-adam-M2-N10}
\end{table*}


We implemented the Adam optimizer in C++.
For the hyper-parameters of Adam, we use
the default values suggested in the original paper~\cite{kingma2015adam}.
The angular velocity $\omg$ is initialized as a zero vector. 
\scenario{IncBat} and \scenario{CopBat} represent the incidence and coplanarity formulations, respectively.
We use \scenario{+exact}, \scenario{+approx} and \scenario{+cascad} to denote methods using exact and approximated parametrizations, as well as their cascade, respectively. When the approximated parametrization is used, the data compression technique is applied.

To evaluate the accuracy of the recovered angular velocities $\omg_{\text{est}}$, we consider both the direction and magnitude. We use the following metric~\cite{guan2024six}
\begin{align}
	\varepsilon_{\text{ang}}(\omg_{\text{est}}; \omg_{\text{gt}}) = \|\omg_{\text{est}} - \omg_{\text{gt}}\| / (\|\omg_{\text{est}}\| + \|\omg_{\text{gt}}\|).
\end{align}
Note that $\|\omg_{\text{est}}\|$, $\|\omg_{\text{gt}}\|$, and $\|\omg_{\text{est}} - \omg_{\text{gt}}\|$ can be considered as the three sides of a triangle. Thus this metric lies within the interval $[0, 1]$, with smaller values indicating better results.

To evaluate the accuracy of the recovered linear velocities, we compare the angular difference between the ground truth and the estimation. Since the estimated translation is only known up to an unknown scale, the metric is 
\begin{align}
	\varepsilon_{\text{lin}}(\v_{\text{est}}; \v_{\text{gt}}) = \arccos \left(  \v_{\text{est}}^\intercal \v_{\text{gt}}  / (\|\v_{\text{est}}\| \cdot \|\v_{\text{gt}}\|) \right).
\end{align}

We conducted our experiments on a laptop with a 12th Generation Intel(R) Core(TM) i7-12700H CPU and $16$ GB RAM. 

\subsection{Simulation}
To thoroughly evaluate our solver’s performance, we first conduct simulation tests.
The ground truth angular velocities $\omega_x$, $\omega_y$, and $\omega_z$ are sampled from a uniform distribution $[-1/8, 1/8]$ rad/s. The ground truth linear velocities $v_x$, $v_y$, and $v_z$ are sampled from a uniform distribution $[-5, 5]$ m/s. The time interval for event set is set to $0.5$ seconds. To synthesize a line $\L$, we first generate a point $\mathbf{p}$ lying on this line by randomly sampling a point from a 3D cube, whose center is $[0, 0, 1]$ and side length is $5$ meters. 
The direction vector $\d$ of each line is random, but only a subset of the lines is retained. These lines must have an angle with the imaging plane that is below $60^\circ$, ensuring that they are somewhat fronto-parallel.
To synthesize an event, we randomly sample a point from the time interval as its imaging time. The scene point corresponding to an event is randomly sampled within a distance of $2.5$ meters from point $\mathbf{p}$. 
The virtual camera has a focal length of $400$ pixels. To extract normal flow, we first accumulate events into frames and then apply the Hough transform to locate event points near the lines. The normal flow of these events corresponds to the normal vectors of the imaged lines.

\begin{figure*}[tbp]
	\centering
	\begin{subfigure}[b]{0.12\linewidth}
		\centering
		\includegraphics[width=\textwidth]{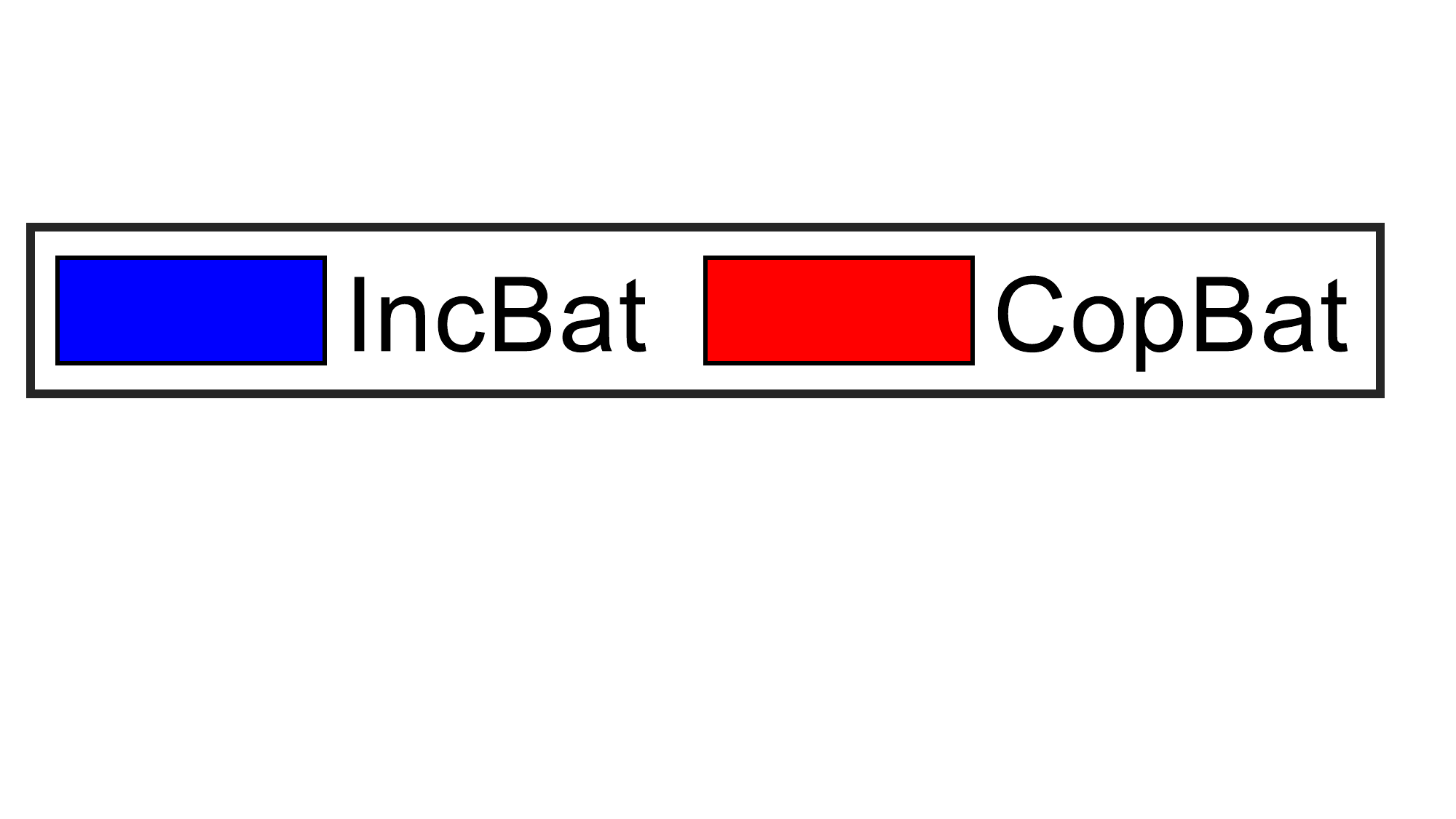}
	\end{subfigure}
	\vspace{-6mm}
	\\ 
	\begin{subfigure}[b]{0.43\linewidth}
		\centering
		\includegraphics[width=\textwidth]{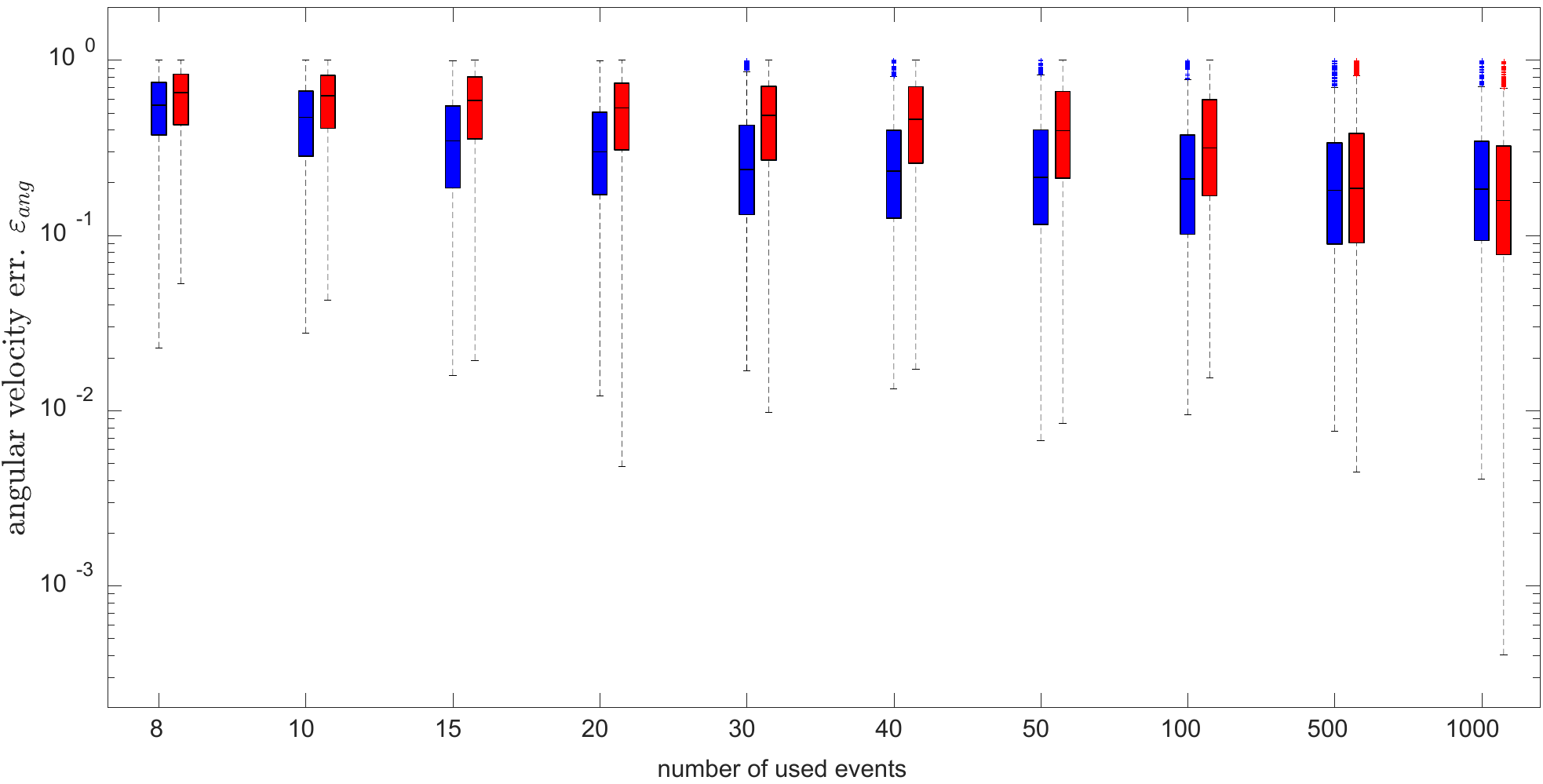}
		\caption{$\varepsilon_{\text{ang}}$ vs. number of events}
		\label{w_err_number_events}
	\end{subfigure}
	\qquad
	\begin{subfigure}[b]{0.43\linewidth}
		\centering
		\includegraphics[width=\textwidth]{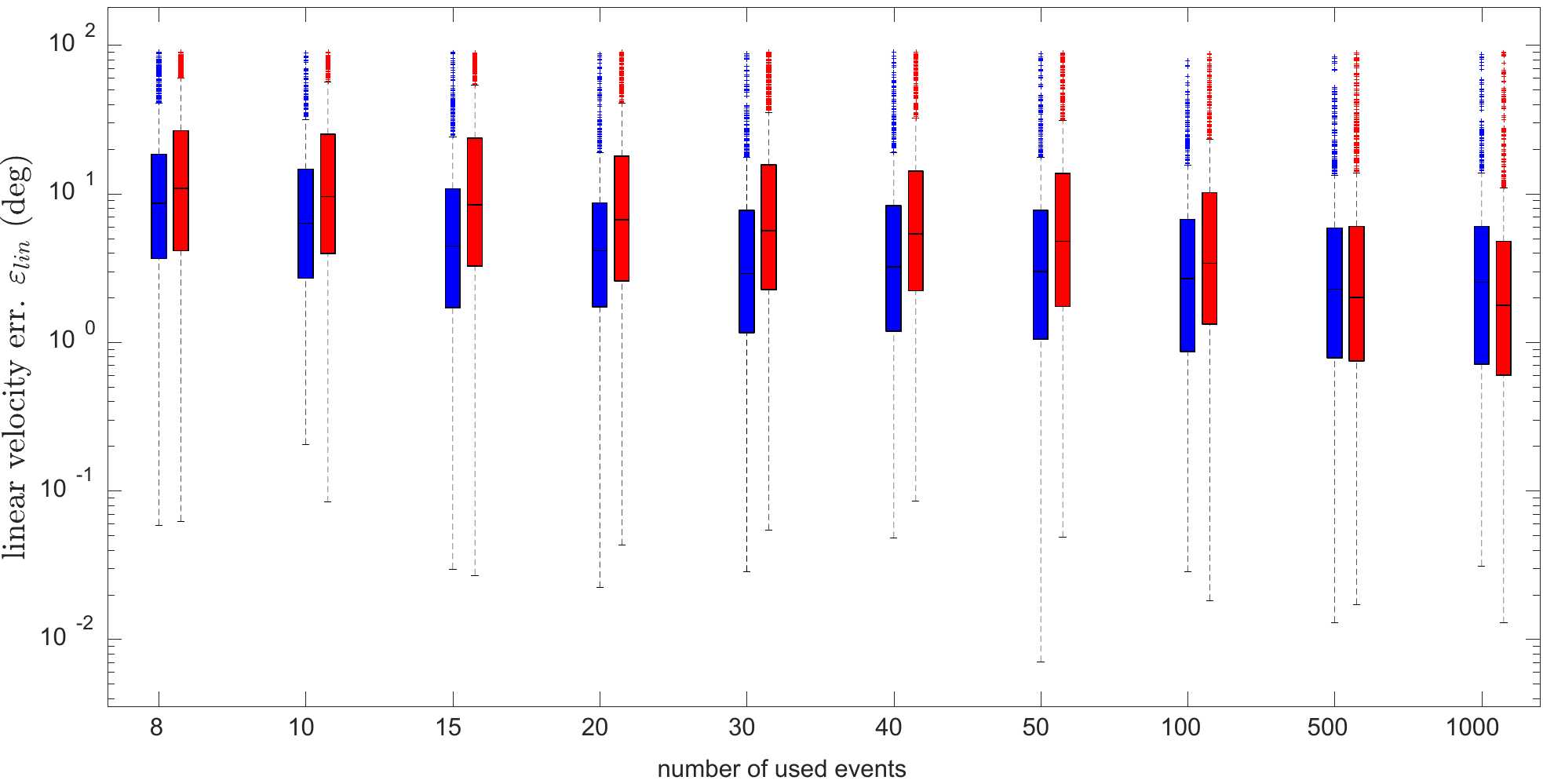}
		\caption{$\varepsilon_{\text{lin}}$ vs. number of events}
		\label{v_err_number_events}
	\end{subfigure}
	\\
	\begin{subfigure}[b]{0.43\linewidth}
		\centering
		\includegraphics[width=\textwidth]{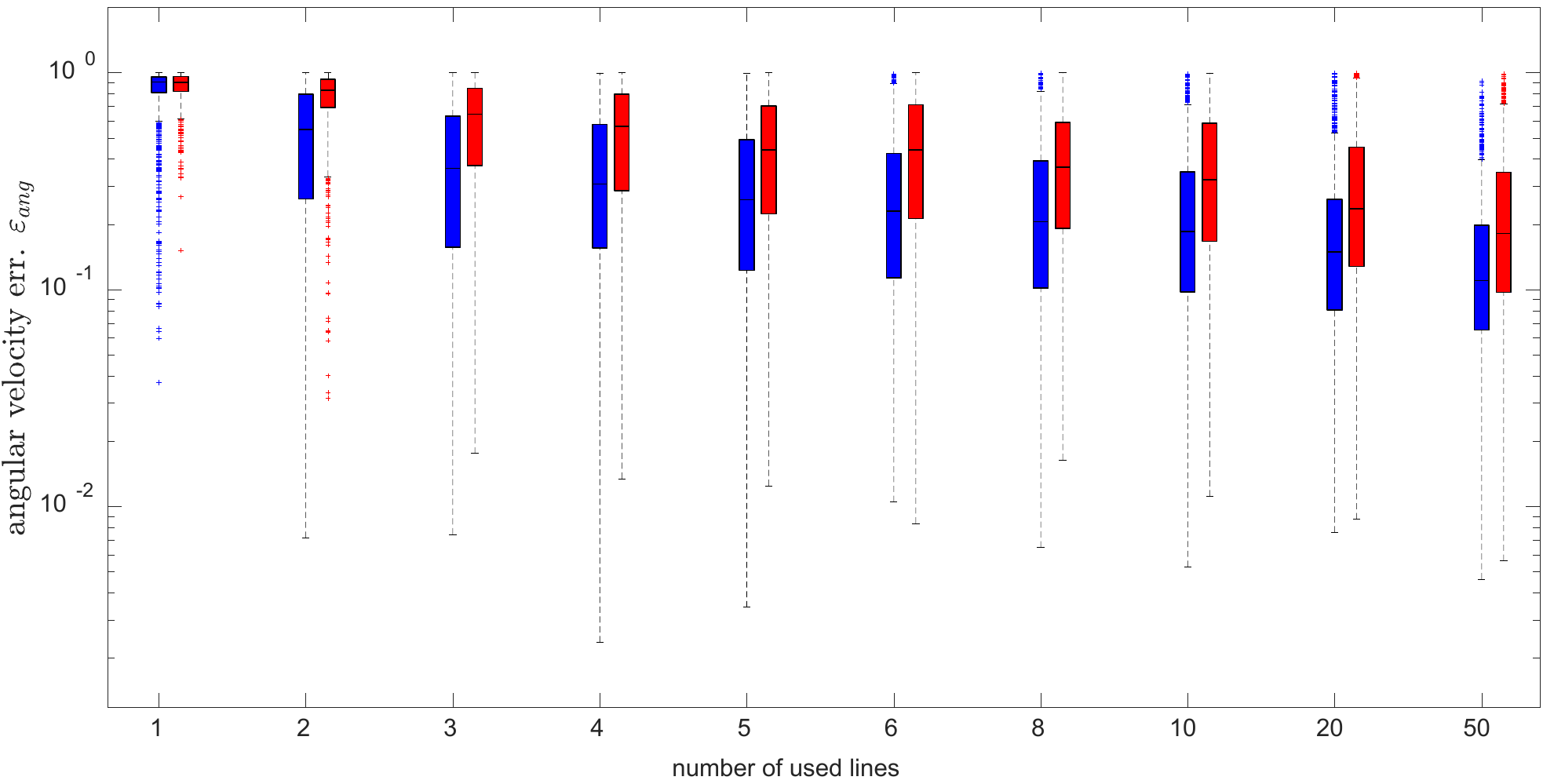}
		\caption{$\varepsilon_{\text{ang}}$ vs. number of lines}
		\label{w_err_number_lines}
	\end{subfigure}
	\qquad
	\begin{subfigure}[b]{0.43\linewidth}
		\centering
		\includegraphics[width=\textwidth]{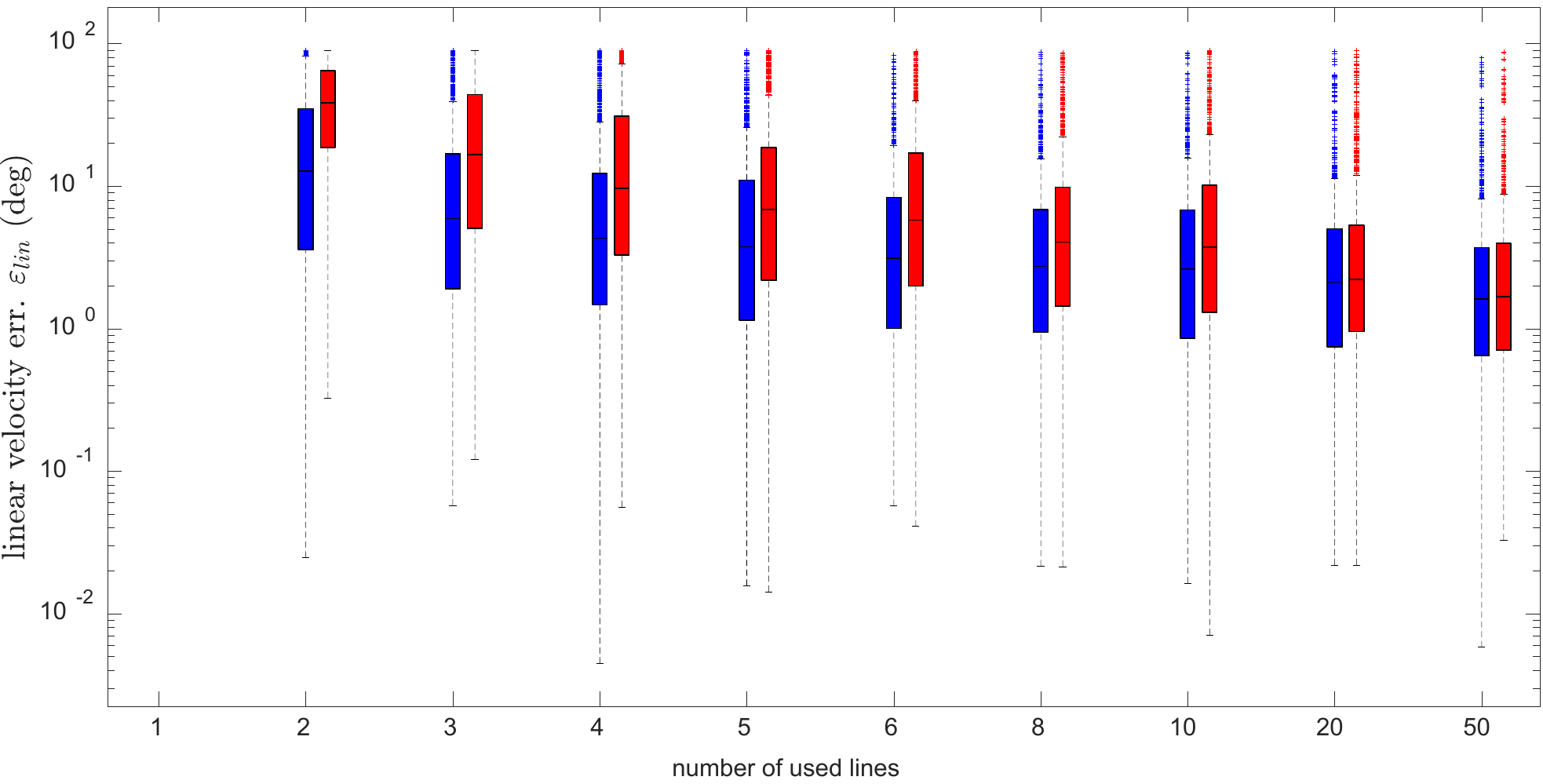}
		\caption{$\varepsilon_{\text{lin}}$ vs. number of lines}
		\label{v_err_number_lines}
	\end{subfigure}
	\\
	\begin{subfigure}[b]{0.43\linewidth}
		\centering
		\includegraphics[width=\textwidth]{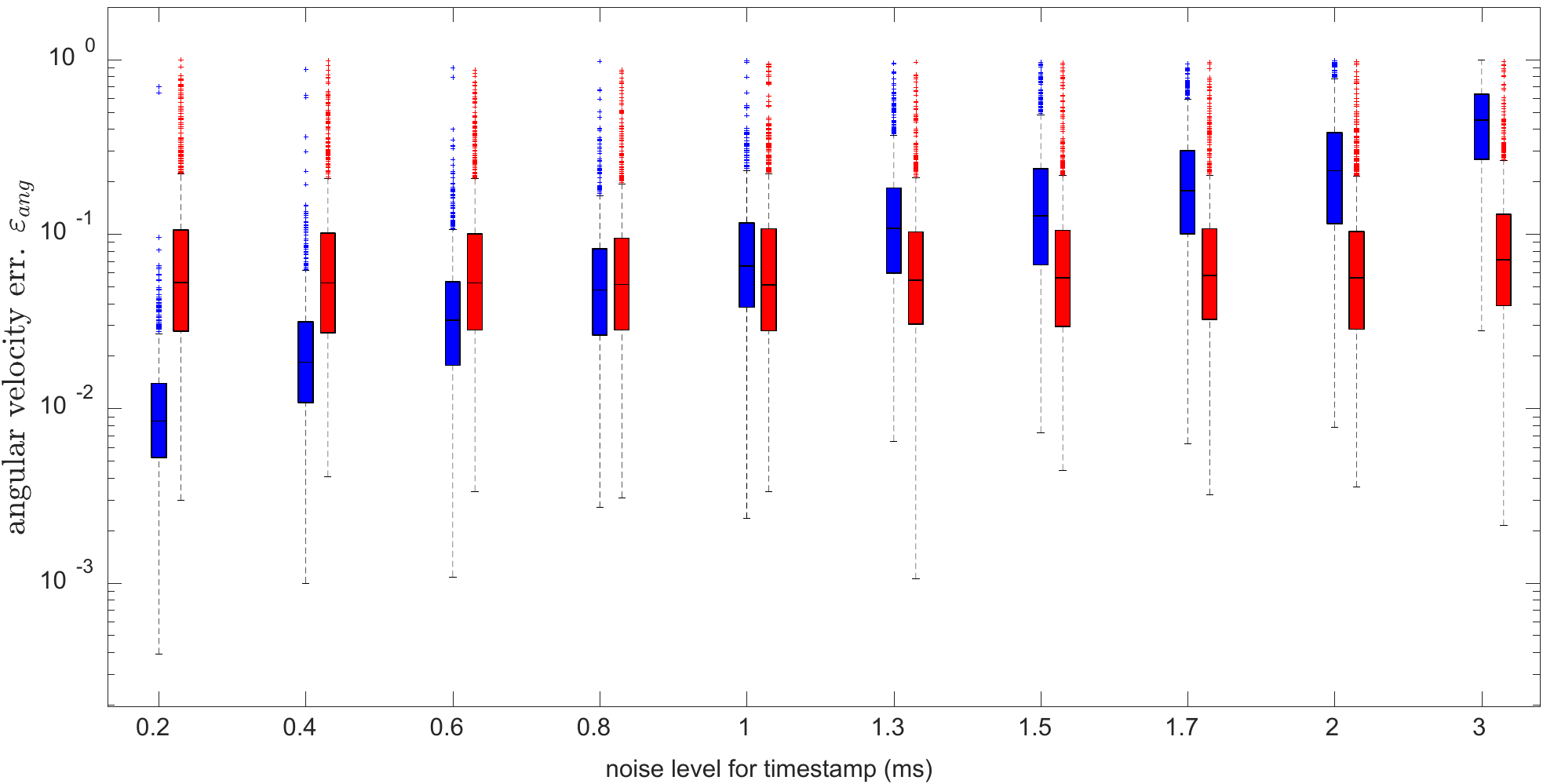}
		\caption{$\varepsilon_{\text{ang}}$ vs. noise level of timestamp}
		\label{w_err_time_noise}
	\end{subfigure}
	\qquad
	\begin{subfigure}[b]{0.43\linewidth}
		\centering
		\includegraphics[width=\textwidth]{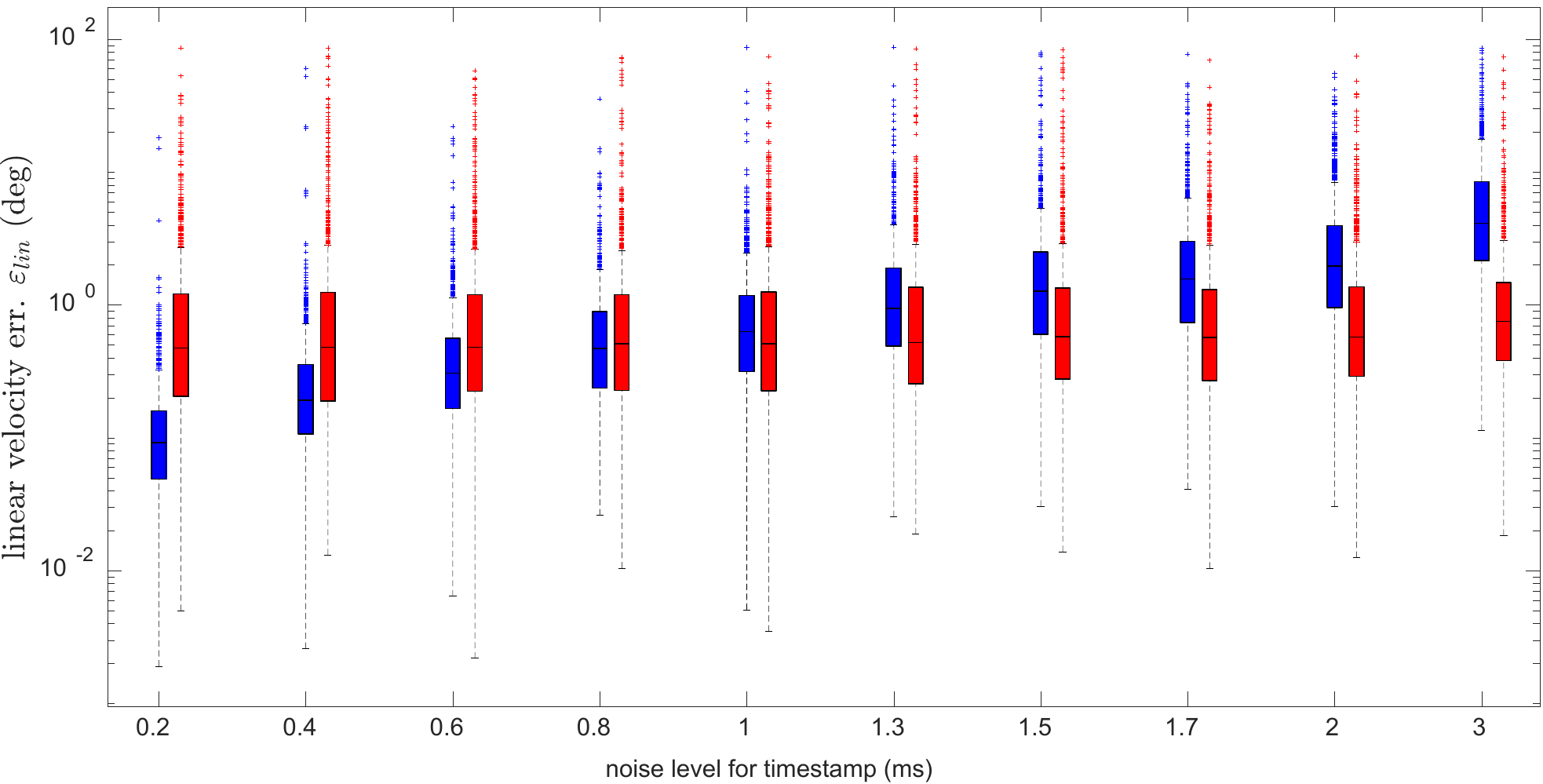}
		\caption{$\varepsilon_{\text{lin}}$ vs. noise level of timestamp}
		\label{v_err_time_noise}
	\end{subfigure}
	\\ 
	\begin{subfigure}[b]{0.43\linewidth}
		\centering
		\includegraphics[width=\textwidth]{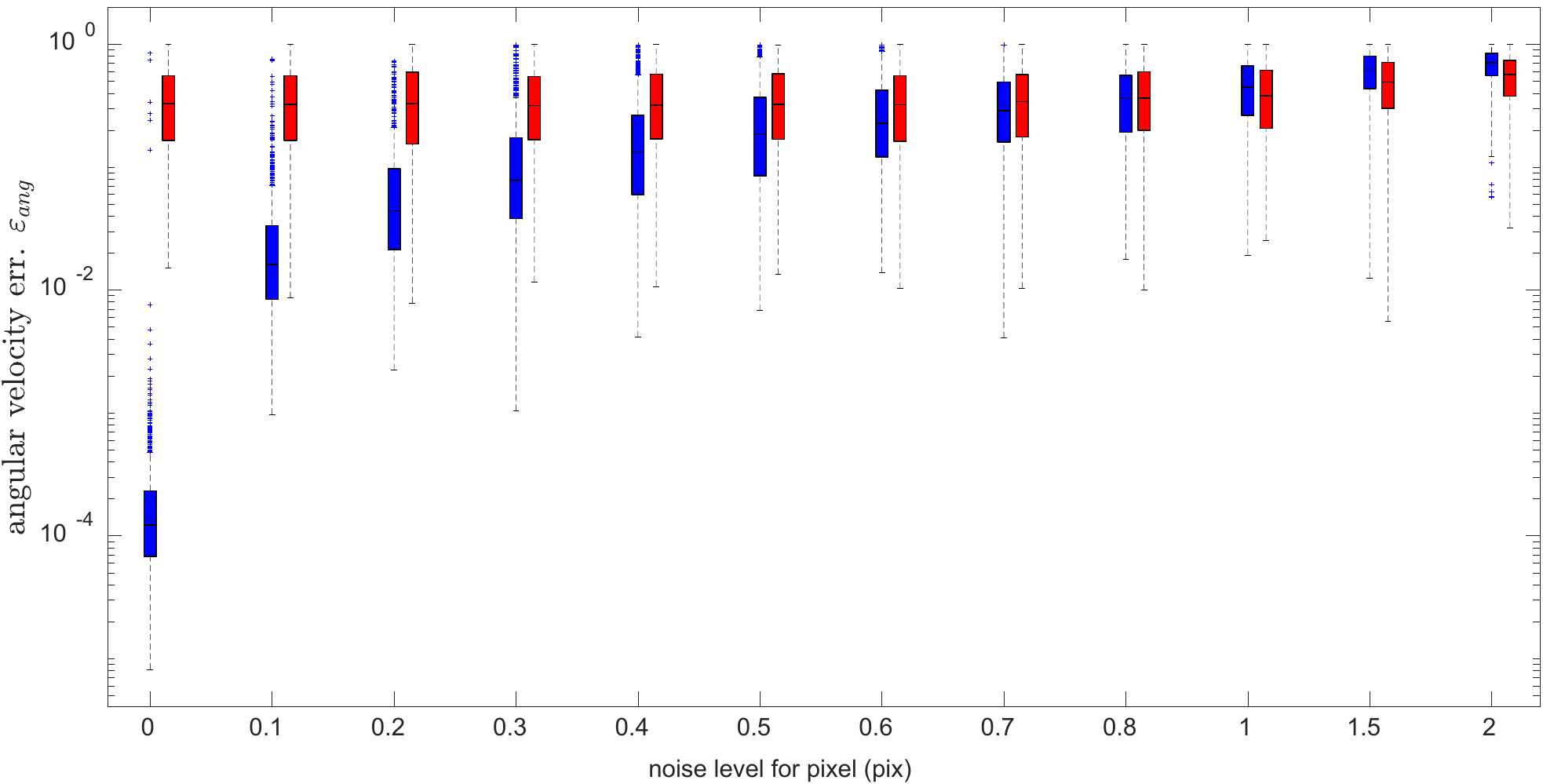}
		\caption{$\varepsilon_{\text{ang}}$ vs. noise level of pixel}
		\label{w_err_pixel_noise}
	\end{subfigure}
	\qquad
	\begin{subfigure}[b]{0.43\linewidth}
		\centering
		\includegraphics[width=\textwidth]{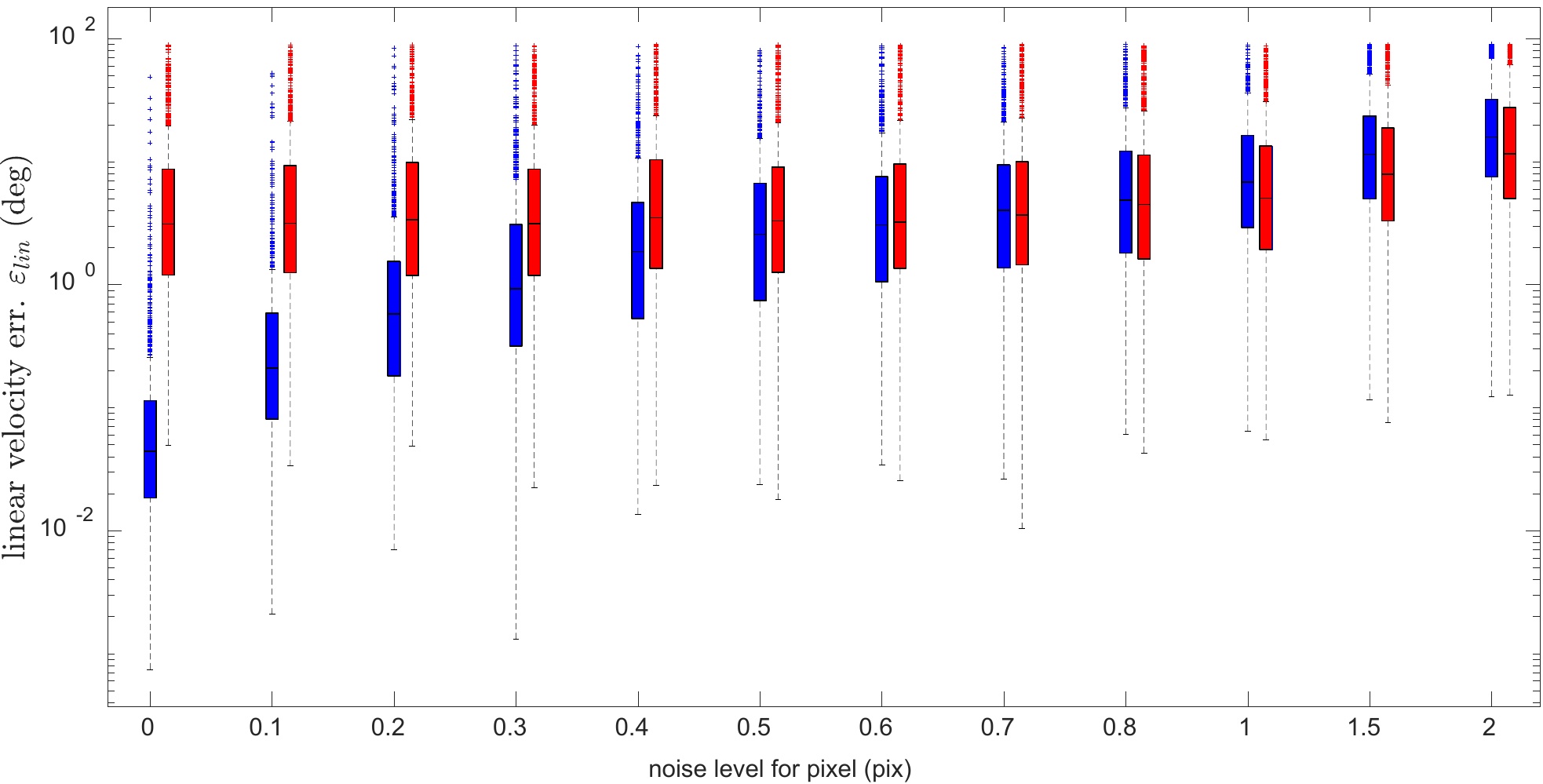}
		\caption{$\varepsilon_{\text{lin}}$ vs. noise level of pixel}
		\label{v_err_pixel_noise}
	\end{subfigure}
	\\ 
	\caption{The results on synthetic data demonstrate the relationship between errors and various factors, such as the number of events, number of lines, and noise levels.}
	\label{fig:synthetic}
\end{figure*}

{\bf Runtime and Numerical Stability.}
To evaluate the efficiency and numerical stability of proposed solvers, we use synthetic data with noise-free observations. 
We generate $1000$ random scenes to evaluate the performance. 
The metrics for a batch of line observations ($M=5$ and $N=100$) are shown in \cref{tab:rslt-adam-M2-N10}. The reported runtime is the median runtime across all tested scenes.
As shown in \cref{tab:rslt-adam-M2-N10}, the runtimes of the proposed solvers are $16.7\sim 48.7$ ms.

We further define the \emph{success rate} for the solvers. 
Angular velocity recovery is considered as successful if the error $\varepsilon_{\text{ang}}$ is below a given threshold. In this experiment, the thresholds to define success are set to $0.01$ and $0.05$. 
From \cref{tab:rslt-adam-M2-N10}, both of the \scenario{IncBat} and \scenario{CopBat} solvers have good accuracy, numerical stability, and success rates. Additionally, the rotation parametrization significantly impacts the performance. The cascade of rotation parameterizations achieves the best accuracy and efficiency. In the following, we will use the cascade method by default.

{\bf Analysis of the Number of Events.}
We set two kinds of noise at different levels during event simulation: pixel noise ($0.5$ pix) and timestamp jitter ($0.5$ ms).
The noise follows a zero-mean Gaussian. 
We fix the the number of lines as $M = 10$ and vary the event number $N$ from $8$ to $1000$ for each line. Each setting is tested $1000$ times with random synthetic data. The accuracy of the results is shown in \cref{fig:synthetic}(a)(b). 
\textcolor{black}{\scenario{IncBat} outperforms \scenario{CopBat} when the event number is below $100$. They tend to have comparable accuracy when increasing event number.} The errors of both solvers decrease rapidly as the number of events increases.

{\bf Analysis of the Number of Lines.}
We set the noise levels for pixel and timestamp as $0.5$ pix and $0.5$ ms, respectively.
We fix the the number of events for each line as $N = 100$ and vary the line number $M$ from $1$ to $50$. Each setting is tested $1000$ times using random synthetic data. The accuracy of the results is shown in \cref{fig:synthetic}(c)(d).
The errors decrease significantly as the number of lines increases. 
When $N=1$, none of the methods perform well. 
This can be attributed to the rotation-translation ambiguity that occurs for a single line cluster.
The errors of both solvers decrease rapidly as the number of lines increases.
The supplementary material provides further analysis and landscape visualization for the objective functions.

{\bf Noise Resilience Analysis.}
We evaluate the noise resilience of the proposed solvers by testing them with varying noise levels in both pixels and timestamps.
The results are shown in \cref{fig:synthetic}(e)-(h).
We can see that errors increase as noise levels increase.
In the absence of noise, our solvers are nearly always successful.
\textcolor{black}{\scenario{IncBat} has smaller errors than \scenario{CopBat} when noise level is small. They tend to have comparable accuracy when noise levels become large.} 

\subsection{Real-World Experiment}
\begin{table}[tbp]
	\centering
	\setlength{\tabcolsep}{7pt} 
	\begin{tabular}{|l|c|c|c|c|}
		\hline
		\multirow{2}{*}{Seq. Name} & \multicolumn{2}{c|}{\scenario{IncBat}} & \multicolumn{2}{c|}{\scenario{CopBat}} \\ \cline{2-5}
		& $\varepsilon_{\text{ang}}$ & $\varepsilon_{\text{lin}}$ ($^\circ$) & $\varepsilon_{\text{ang}}$ & $\varepsilon_{\text{lin}}$ ($^\circ$) \\ \hline
		\emph{desk-normal} & $0.232$ & $23.0$ & $0.236$ & $25.1$ \\ \hline
		\emph{mountain-normal} & $0.195$ & $17.5$ & $0.216$ & $18.7$ \\ \hline
		\emph{sofa-normal} & $0.229$ & $21.1$ & $0.221$ & $20.6$ \\ \hline
	\end{tabular}
	\caption{Real-world experiment results. We report the median errors for $\varepsilon_{\text{ang}}$ and $\varepsilon_{\text{lin}}$.}
	\label{tab:real}
\end{table}
 
We use three sequences from the VECtor dataset~\cite{gao2022vector}, which contain clear line structures. The event sequences are segmented into non-overlapping intervals of $0.3$ seconds each. We select intervals that approximately satisfy the constant velocity assumption.
Line clusters are extracted using the Hough transform, and normal flow is computed using \cite{benosman2013event,hordijk2018vertical}.
The GC-RANSAC~\cite{barath2022graph} is used as the robust estimation framework. 

The accuracy of recovered angular and linear velocities is reported in \cref{tab:real}.
The results indicate that the proposed solver is applicable to real-world data. Such results are typically sufficient for integration into a VIO/SLAM pipeline, as demonstrated in \cite{peng2021continuous,xu2023tight}.
More experimental settings and results are provided in the supplementary material.

\section{Conclusion}
\label{sec:conclusion}

In this paper, we have filled a crucial gap in egomotion estimation for event cameras by developing sparse geometric solvers. They are capable of recovering full-DoF egomotion parameters—both rotational and translational velocities—without relying on additional sensors. Our method uses event manifolds from line segments, employing incidence relations or a novel coplanarity relation. By leveraging the Adam optimizer and the first-order rotation approximations for efficient initialization, we demonstrate effective performance of the proposed solvers on synthetic and real-world datasets.

\noindent\textbf{Acknowledgments}
{\small
This research has been supported by the National Natural Science Foundation of China (Grant No. 12372189) and the Hunan Provincial Natural Science Foundation for Excellent Young Scholars (Grant No. 2023JJ20045).
The authors also acknowledge the funding support provided by projects 22DZ1201900, 22ZR1441300, and DFYJBJ-1 by the Natural Science Foundation of Shanghai.
}

{
    \small
    \bibliographystyle{ieeenat_fullname}
    \bibliography{main}
}

\clearpage
\maketitlesupplementary
\appendix
\section*{Appendix}

\section{Optimization}

The comparison of the proposed optimization problems is summarized in \cref{tab:opt-problem-comp}.
\begin{table}[h]
	\centering
	\setlength{\tabcolsep}{4pt} 
	\begin{tabular}{lccccr}
		\toprule
		Problem & Geometry & obs. & \#line & var. & matrix \\
		\midrule
		\eqref{eq:opt1} & incidence & raw & $1$ & $\omg$ & $\M_{6\times 6}$ \\
		\eqref{eq:opt1-i} & incidence & raw & $\ge 2$ & $\omg$ & $(\M_i)_{6\times 6}$ \\
		\midrule
		\eqref{eq:opt_lmd_nf} & coplanarity & flow & $1$ & $\omg$ & $\N_{3\times 3}$ \\
		\eqref{eq:opt-lmd_nf_i} & coplanarity & flow & $\ge 2$ & $\omg$ & $(\N_i)_{3\times 3}$ \\
		\bottomrule
	\end{tabular}
	\caption{Comparison of optimization problems. obs: observations; var: variable; raw: raw events; flow: normal flow; Min: minimal configuration (a single line); Bat: a batch of lines.}
	\label{tab:opt-problem-comp}
\end{table}

\subsection{Remark}

\hspace{1em}{\bf Remark 1}: 
In optimization problems \eqref{eq:opt1-i} \eqref{eq:opt-lmd_nf_i} with a batch of line observations, the objectives are summation of minimal eigenvalues. It is also possible to replace $\lambda_{\text{min}}$ by a monotonically increasing function with respect to $\lambda_{\text{min}}$. Usually, we can use $\lambda^p_{\text{min}}$, where $p$ is an exponent. Popular selection of $p$ includes $1$, $2$, and $1/2$. The optimal selection of $p$ is determined by the noise level of observations. In this work, we simply set $p$ as $1$.

{\bf Remark 2}:
Given ground truth $\omg_{gt}$ and exact rotation parametrization, the smallest eigenvalue $\lambda_{\text{min}}(\M_i(\omg))$ is about $10^{-16}$. Given ground truth $\omg_{gt}$ and approximated rotation parametrization, $\lambda_{\text{min}}(\M_i(\omg))$ is about $10^{-12} \sim 10^{-9}$. 
Since the smallest eigenvalues of $\M_i(\omg)$ in our problems are significantly small, the floating number calculation causes insufficient accuracy and the termination conditions of optimization methods are easily triggered. To mitigate these issues, we multiply matrices $\{\M_i(\omg)\}_i$ by a factor of $10^{6}$. This does not have any influence on the optimum of $\omg$, and it only increases the objective $\lambda_{\text{min}}$ by a certain factor.

{\bf Remark 3}:
If an event has an associated weight, we can simply multiply the weight by the corresponding row of the matrix $\A$ or $\B$.


\section{Egomotion Estimation for Pure Rotation Cases}
\label{sec:pure-rot}

Pure rotational motion is often a degenerate case in relative pose estimation, especially for sparse geometric solvers. In this section, we will discuss how to deal with it.

\begin{proposition}
	Pure rotation leads to 
	\begin{align}
		\rank([\f'_1, \cdots, \f'_N]) = 2
		\label{eq:pure-rot}
	\end{align}
	for unique events $\{\f'_j\}_{j=1}^N$ of a 3D line. 
	Bearing vectors $\{\f'_j\}_{j=1}^N$ lies in a same plane whose normal is $\e_2$.
	\label{pro:pure-rot}
\end{proposition}
\begin{proof}
	Since the motion is pure rotation, the origin of any $\f'_j$ is the optical center. Meanwhile, any $\f'_j$ intersects with a 3D line. Thus bearing vectors $\{\f'_j\}_{j=1}^N$ lie on a same plane spanned by the optical center and the 3D line whose normal is $\e_2$. 
\end{proof}
When rotation is known, \cref{eq:pure-rot} provides a criterion to determine whether the motion is pure rotation, i.e., linear velocity $\v = \mathbf{0}$. 
We can identify the pure rotation cases and immediately know the linear velocity $\v$ is zero.

\begin{proposition}
	When pure rotation occurs, we have $\rank(\A) = \rank(\M) = 4$. The null space of $\A$ has two orthogonal basis $[\e_2; 0; 0; 0]$ and $[0; 0; 0; \e_2]$.
	\label{pro:null-space}
\end{proposition}
\begin{proof}
	According to Proposition~\ref{pro:pure-rot}, bearing vectors $\{\f'_j\}_{j=1}^N$ lies in a same plane whose normal is $\e_2$, so does the vector set $\{t_j \f'_j\}_{j=1}^N$. Note that $[t_1 \f'_1, \cdots, t_N \f'_N]$ and $[\f'_1, \cdots, \f'_N]$ are the transpose of $\A_{:, 1:3}$ and $\A_{:, 4:6}$. In summary, both the left and right three columns of $\A$ has a null space basis $\e_2$. So $\A$ has two null space basis $[\e_2; 0; 0; 0]$ and $[0; 0; 0; \e_2]$.
	Since $\M = \A^\intercal \A$, $\M$ has the same rank as that of $\A$.
\end{proof}

In \cite{gao2024npoint}, the proposed solver can only deal with cases when $\rank(\A) \ge 5$. In Proposition~\ref{pro:null-space}, we have proved that $\rank(\A) = 4$. 
Thus it cannot deal with the pure rotation cases.
Our pure rotation identification method provides a supplement to the solver in~\cite{gao2024npoint}.
In addition, we can use  $\A_{:, 1:3}$ or $\A_{:, 4:6}$ to recover $\e_2$. For $\e_1$ and $\e_3$, they span a plane whose normal is $\e_2$, but they cannot be uniquely determined. So the complete line parameters cannot be recovered.

Additionally, we can recover the angular velocity $\omg$ for pure rotation cases within the incidence formulation.
According to Proposition~\ref{pro:null-space}, the top-left and bottom-right $3\times 3$ minor of $\M$ have a rank deficiency. So the angular velocity $\omg$ for pure rotation cases can be optimized by 
\begin{align}
	\argmin_{\omg} \ \sum_{i=1}^M \lambda_{\text{min}}(\widehat{\M}_i(\omg))
	\label{eq:opt1-pure-rot}
\end{align}
where $\widehat{\M}$ is bottom-right $3\times 3$ minor of $\M$, i.e., $\widehat{\M} = \M_{4:6, 4:6}$.

The proposed coplanarity formulation can recover the angular velocity $\omg$ natively for pure rotation cases.

\section{Experiments}

\begin{table*}[t]
	\centering
	\begin{tabular}{|l|l|r|r|r|r|r|}
		\hline
		formulation & rotation  & median $\varepsilon_{\text{ang}}$ & SR1(\%) & SR2(\%) & median objective & runtime \\ \hline
		\multirow{3}{*}{\scenario{IncMin}}
		& \scenario{+approx} & $3.0\times 10^{-1}$  & $0.8$ & $9.6$ & $5.6\times 10^{-11}$ & $9.6$ ms \\ 
		& \scenario{+exact} & $2.9\times 10^{-1}$  & $2.0$ & $10.2$ & $5.4\times 10^{-11}$ & $15.8$ ms \\ 
		& \scenario{+cascad} & $2.8\times 10^{-1}$  & $2.5$ & $11.0$ & $4.0\times 10^{-11}$ & $15.8$ ms \\ 
		\hline
		\multirow{3}{*}{\scenario{CopMin}} 
		& \scenario{+approx} & $5.8\times 10^{-1}$  & $0.0$ & $0.8$ & $4.3\times 10^{-10}$ & $15.7$ ms \\ 
		& \scenario{+exact} & $5.7\times 10^{-1}$  & $0.1$ & $1.3$ & $4.4\times 10^{-10}$ & $15.9$ ms \\ 
		& \scenario{+cascad} & $5.5\times 10^{-1}$  & $0.0$ & $0.7$ & $2.2\times 10^{-10}$ & $15.8$ ms \\ 
		\hline
	\end{tabular}
	\caption{Runtime and numerical stability for noise-free synthetic data. The configuration is $M=1$ and $N=100$. SR1 and SR2 represent success rate (SR) with thresholds of $0.01$ and $0.05$, respectively.
	\scenario{IncMin} and \scenario{CopMin} represent the incidence formulation and coplanarity formulation, respectively.}
	\label{tab:rslt-adam-M1-N8}
\end{table*}

\subsection{Simulation}

In the simulation tests, we observed that the recovered angular velocities were unsatisfactory when only a single line ($M = 1)$ was available. To investigate the underlying reasons for this phenomenon, we conducted additional experiments.

The runtime and numerical stability for the minimal configuration ($M=1$ and $N=100$) are presented in \cref{tab:rslt-adam-M1-N8}.
While the objective values $\lambda_{\text{min}}$ are quite small, the accuracy of all solvers remains inadequate, as evidenced by the significantly low success rates. This phenomenon can be intuitively explained.
With only one line, it is possible to orbit around that line while simultaneously adjusting the translational velocity to compensate for the added rotational velocity about an axis parallel to the line. This creates a form of rotation-translation ambiguity. Although this ambiguity might not be perfect, the constant velocity assumption in our motion model makes approximations in any case. Thus, it is reasonable to conclude that such ambiguity significantly impacts the performance.

To further validate this observation, we visualize the objective landscapes in \cref{fig:landscape-line}.
When only a single line is used, the objective functions lack a clear global minimum.
When two lines are used, the objective functions still have relatively large convergence regions.
However, when more than three lines are included, the objective functions exhibit distinct global minima.

\begin{figure*}[tbp]
	\centering
	\begin{subfigure}[b]{0.3\linewidth}
		\centering
		\includegraphics[width=\textwidth]{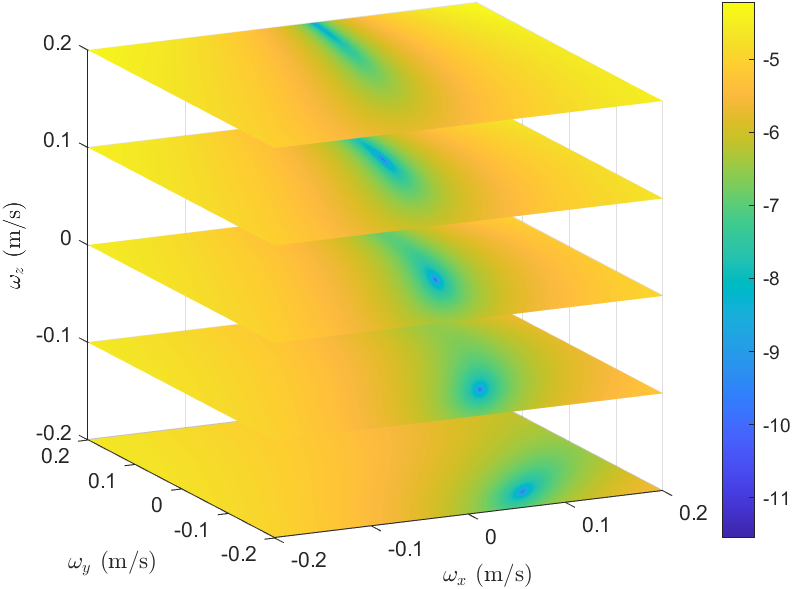}
		\caption{\scenario{IncMin} with $1$ line}
	\end{subfigure}
	\qquad
	\begin{subfigure}[b]{0.3\linewidth}
		\centering
		\includegraphics[width=\textwidth]{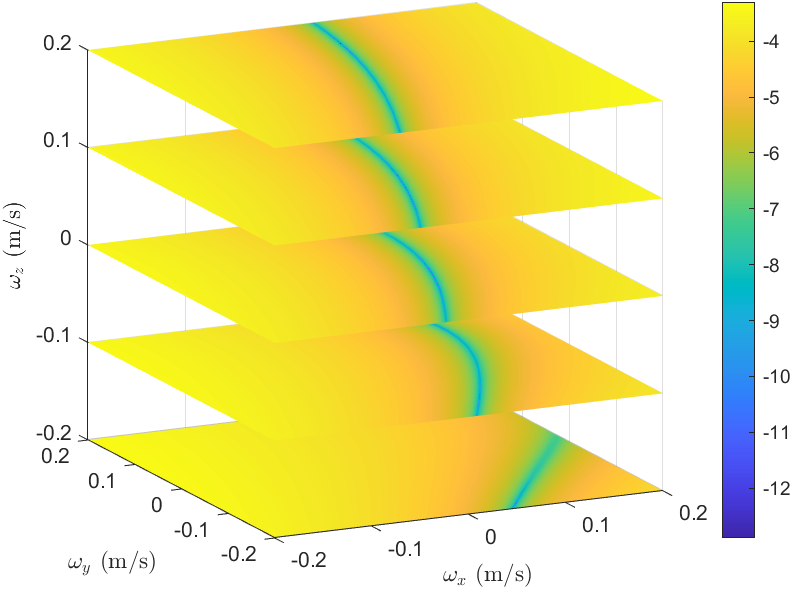}
		\caption{\scenario{CopMin} with $1$ line}
	\end{subfigure}
	\\
	\begin{subfigure}[b]{0.3\linewidth}
		\centering
		\includegraphics[width=\textwidth]{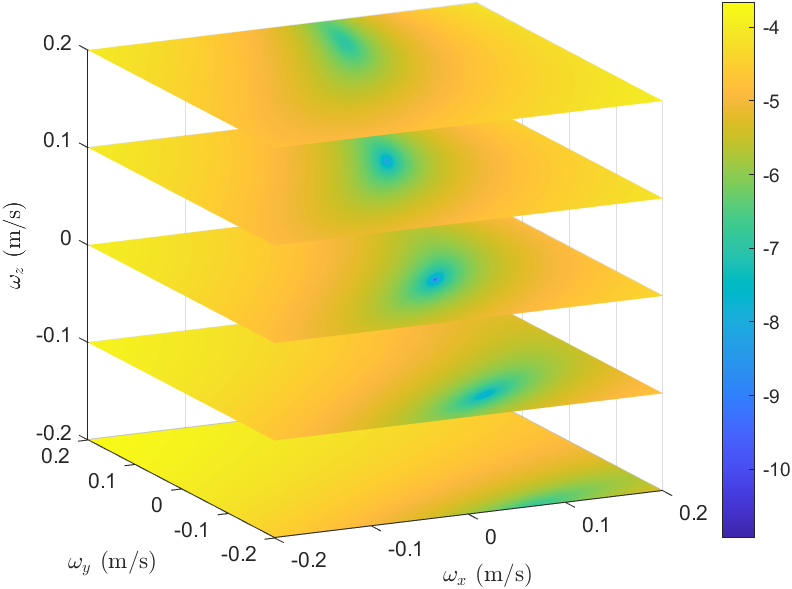}
		\caption{\scenario{IncBat} with $2$ lines}
	\end{subfigure}
	\qquad
	\begin{subfigure}[b]{0.3\linewidth}
		\centering
		\includegraphics[width=\textwidth]{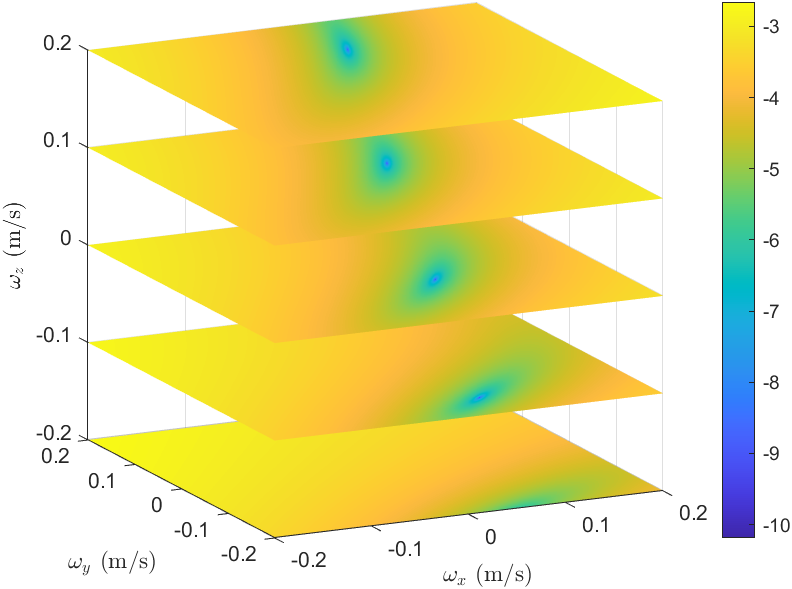}
		\caption{\scenario{CopBat} with $2$ lines}
	\end{subfigure}
	\\
	\begin{subfigure}[b]{0.3\linewidth}
		\centering
		\includegraphics[width=\textwidth]{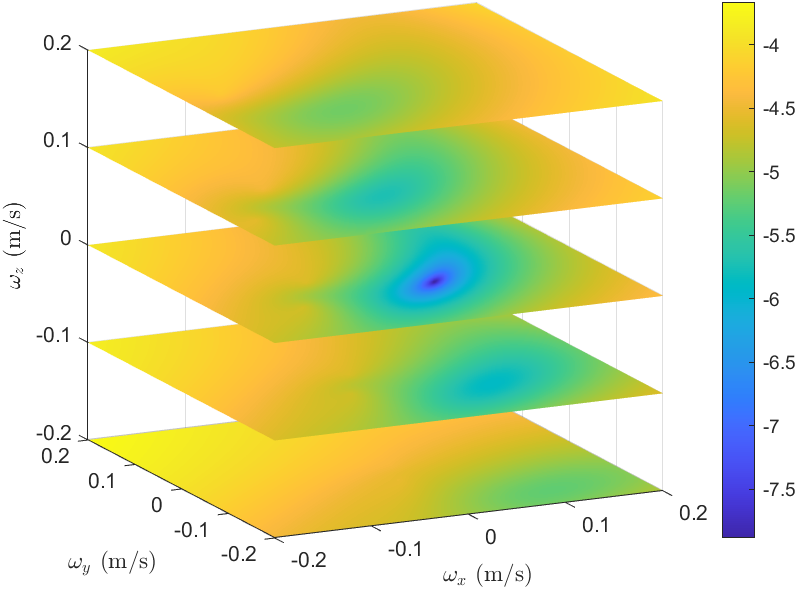}
		\caption{\scenario{IncBat} with $3$ lines}
	\end{subfigure}
	\qquad
	\begin{subfigure}[b]{0.3\linewidth}
		\centering
		\includegraphics[width=\textwidth]{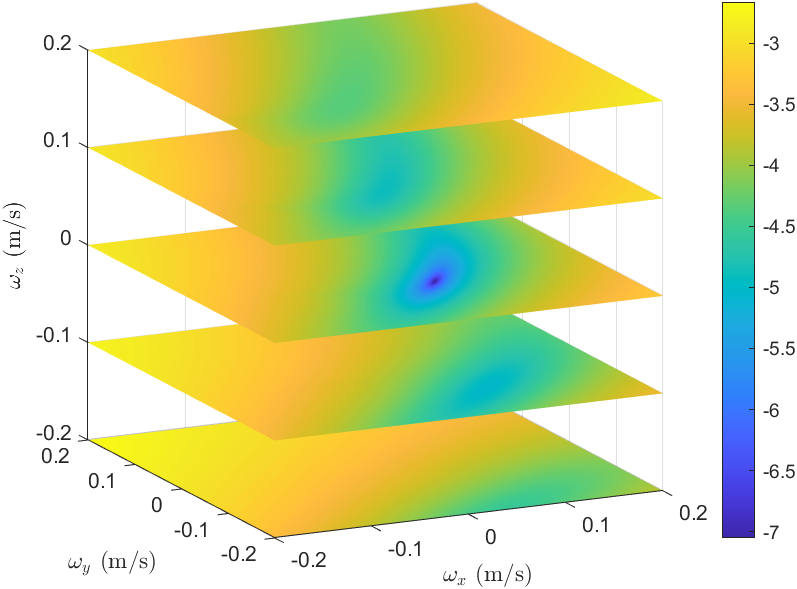}
		\caption{\scenario{CopBat} with $3$ lines}
	\end{subfigure}
	\\ 
	\begin{subfigure}[b]{0.3\linewidth}
		\centering
		\includegraphics[width=\textwidth]{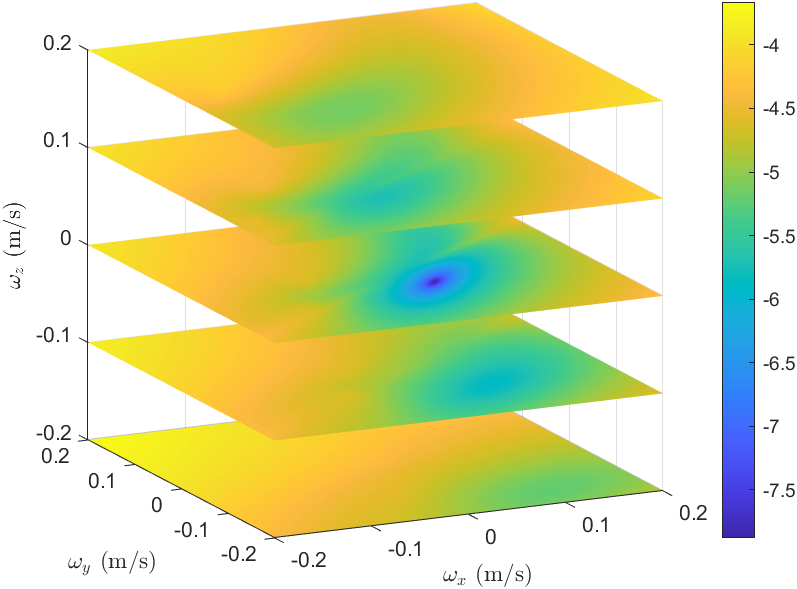}
		\caption{\scenario{IncBat} with $4$ lines}
	\end{subfigure}
	\qquad
	\begin{subfigure}[b]{0.3\linewidth}
		\centering
		\includegraphics[width=\textwidth]{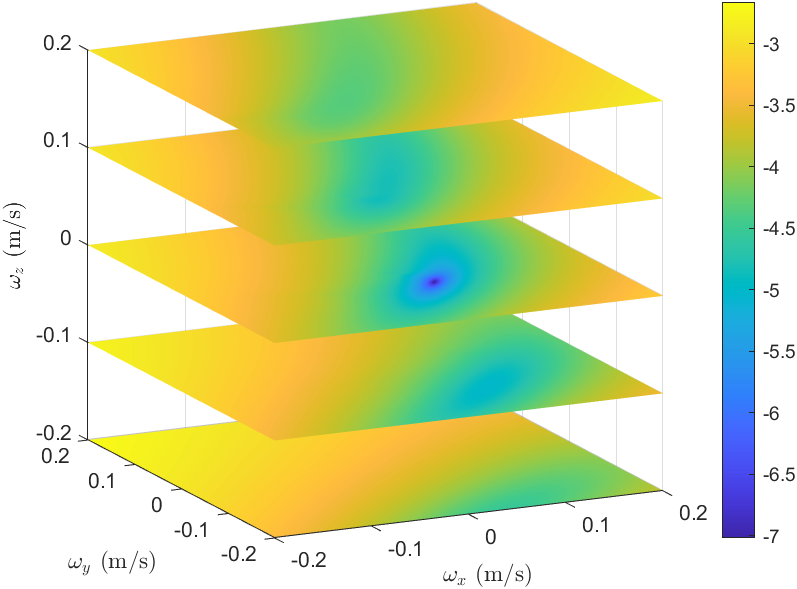}
		\caption{\scenario{CopBat} with $4$ lines}
	\end{subfigure}
	\\
	\begin{subfigure}[b]{0.3\linewidth}
		\centering
		\includegraphics[width=\textwidth]{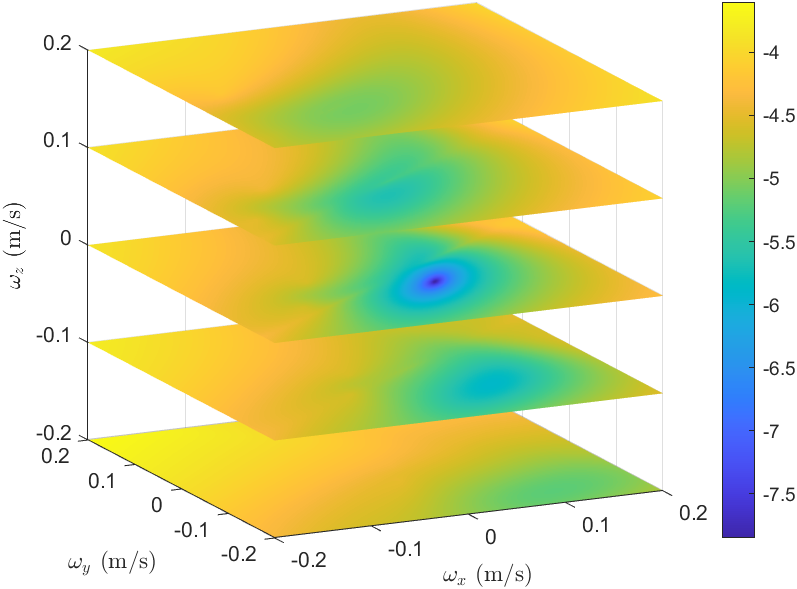}
		\caption{\scenario{IncBat} with $5$ lines}
	\end{subfigure}
	\qquad
	\begin{subfigure}[b]{0.3\linewidth}
		\centering
		\includegraphics[width=\textwidth]{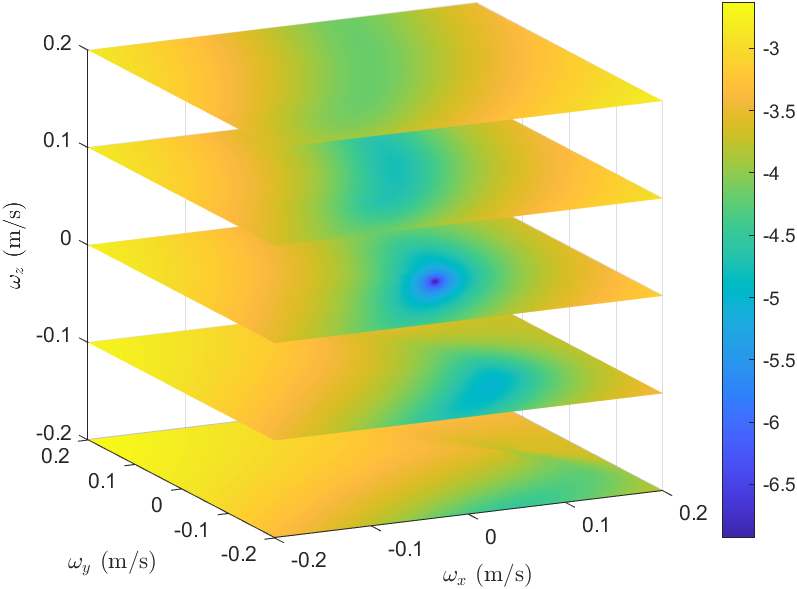}
		\caption{\scenario{CopBat} with $5$ lines}
	\end{subfigure}
	\\
	\caption{Landscape of the objective functions $\lambda_{\text{min}}$. The events for each line is set as $N = 100$. For better visualization, the pseudocolor and colormap of the objectives use the logarithmic scale.}
	\label{fig:landscape-line}
\end{figure*}

In the synthetic experiments, the success rates of solvers do not reach $100\%$ even without noise. There are several reasons for this.
First, for the approximative rotation parametrization, an approximation of the objective is implicitly introduced. As a result, even its global minimum may deviate slightly from the ground truth. This also explains why we used the results from approximative rotations as an initialization of the cascade method. Second, the objectives are non-convex, and the local optimization method may get trapped in local minima. Third, the objective's landscape may be relatively flat near the global minimum, causing local optimization methods to converge slowly or fail entirely to converge. This explains why the success rate varies when different thresholds are used to define success.

\subsection{Real-World Experiment}

As our solvers are the first capable of estimating both rotational and translational parameters, we evaluated their applicability on the VECtor dataset~\cite{gao2022vector}. This dataset includes VGA-resolution event recordings captured by a Gen3 Prophesee camera, $200$ Hz ground truth camera poses obtained via a MoCap system, and $200$ Hz readings from an XSens MTi-30 AHRS IMU. To improve efficiency, the events within each time interval are downsampled to $5000$.

\begin{figure*}[tbp]
	\centering
	\begin{subfigure}[b]{0.4\linewidth}
		\centering
		\fbox{\includegraphics[width=\textwidth]{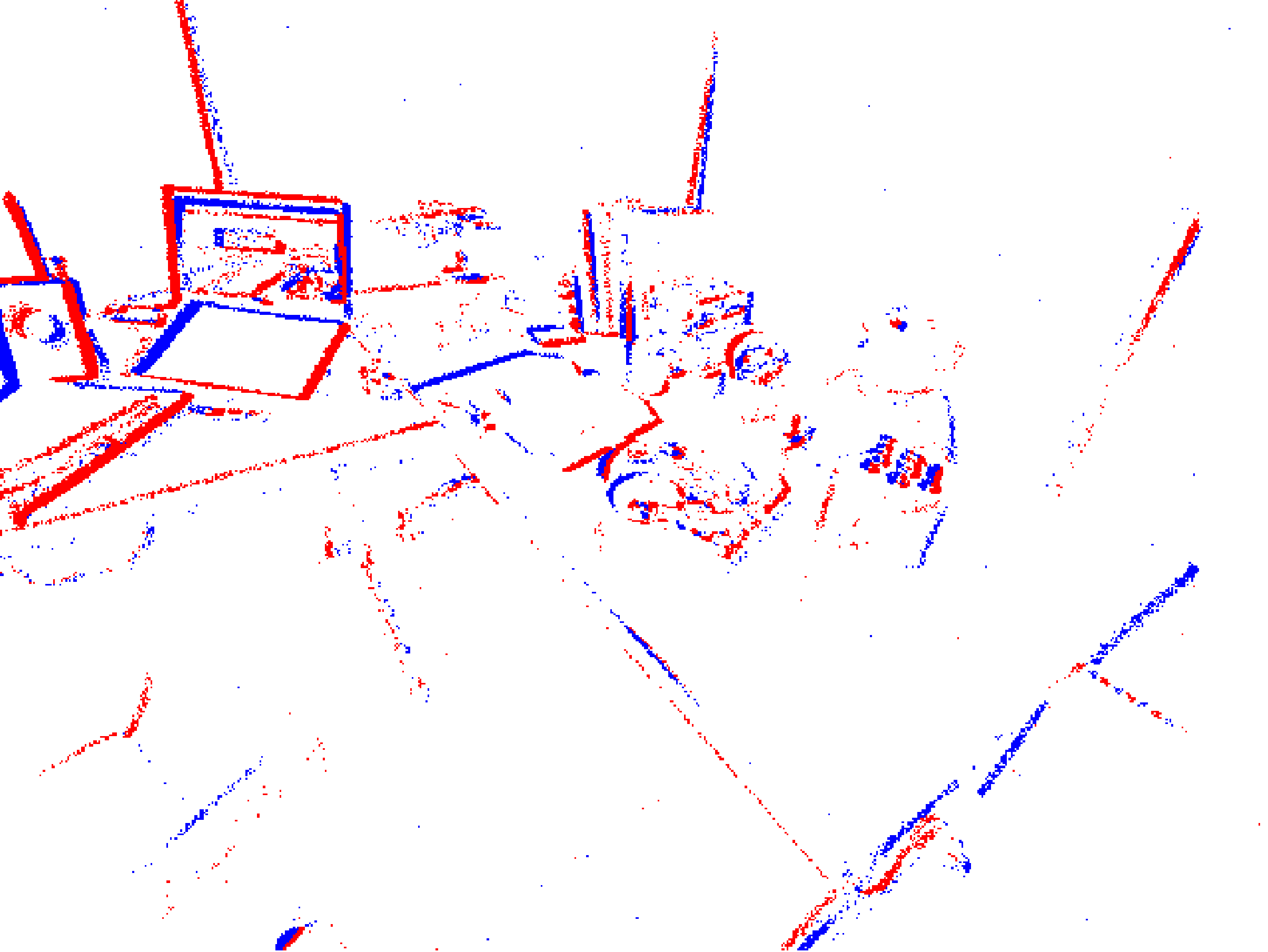}}
		\caption{An event frame}
	\end{subfigure}
	\qquad \quad \ 
	\begin{subfigure}[b]{0.375\linewidth}
		\centering
		\includegraphics[width=\textwidth]{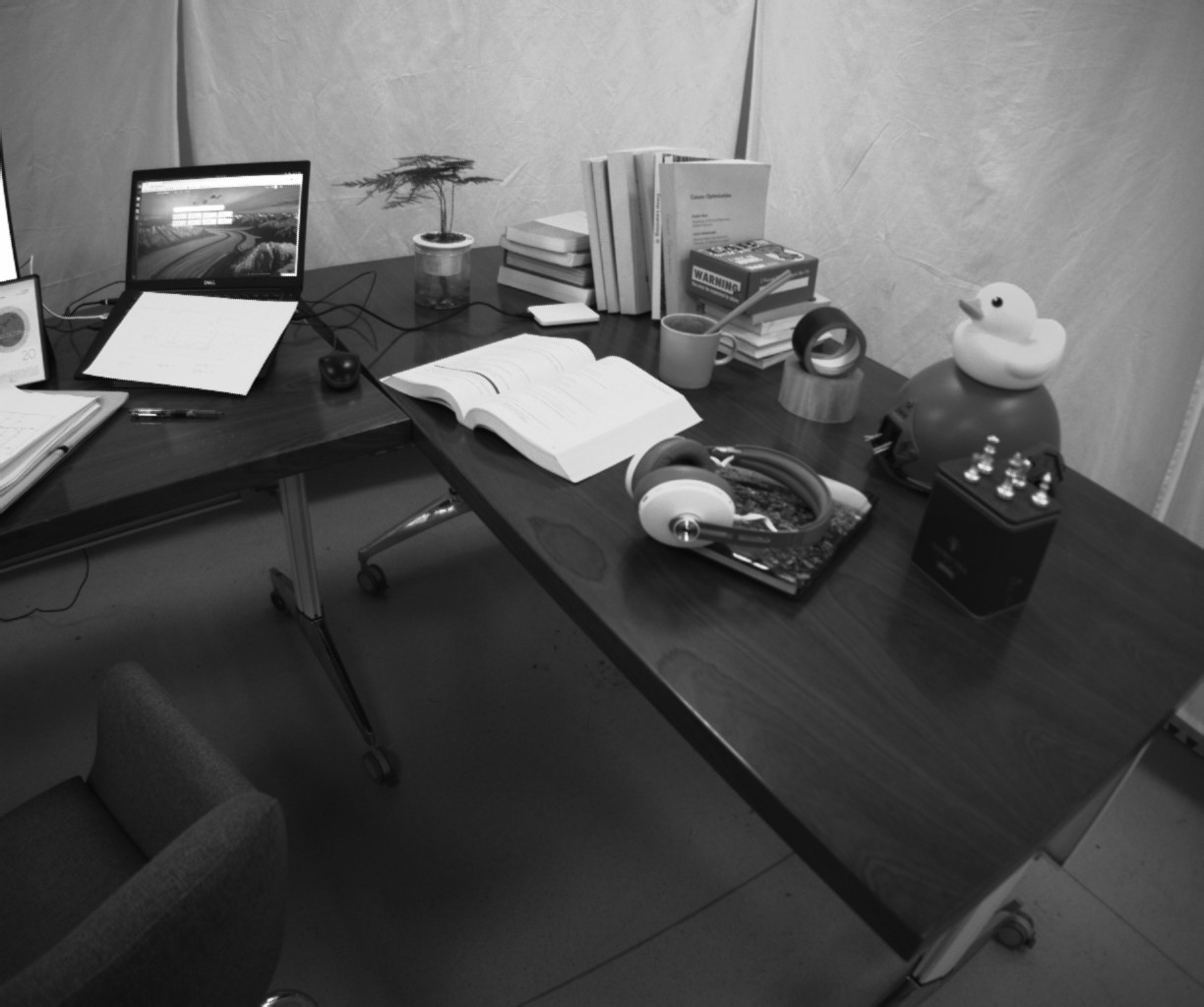}
		\caption{image. It is used for visualization only.}
	\end{subfigure}
	\\
	\vspace{1em}
	\begin{subfigure}[b]{0.4\linewidth}
		\centering
		\fbox{\includegraphics[width=\textwidth]{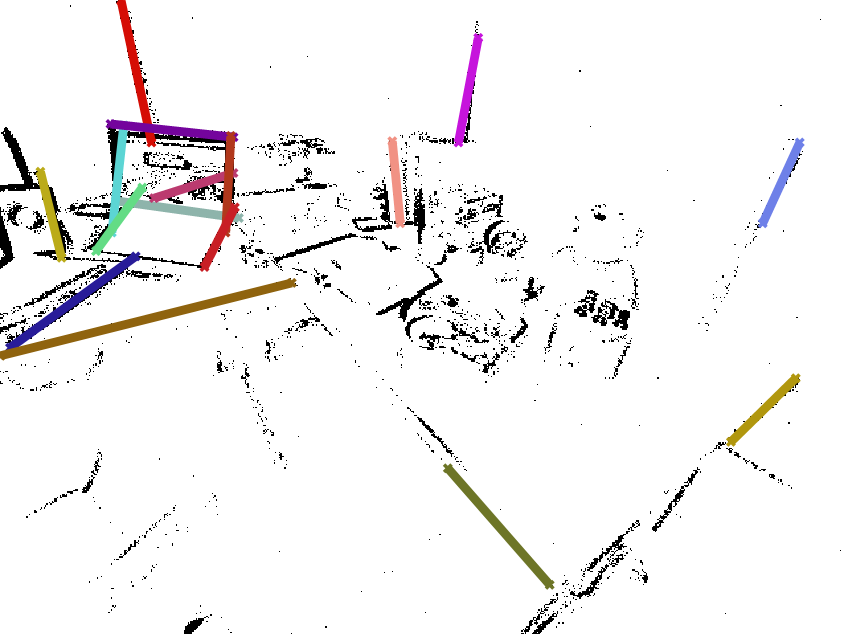}}
		\caption{line segment detection}
	\end{subfigure}
	\qquad
	\begin{subfigure}[b]{0.4\linewidth}
		\centering
		\fbox{\includegraphics[width=\textwidth]{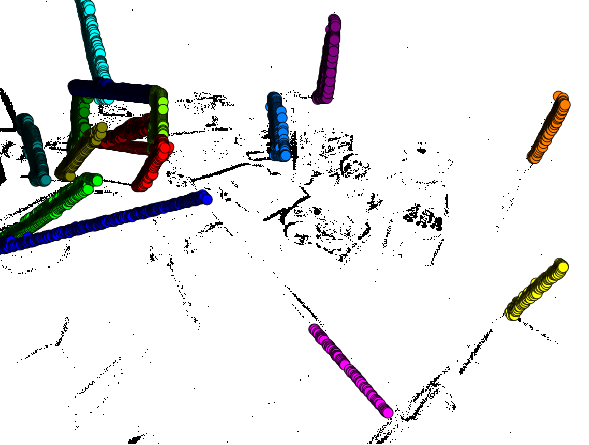}}
		\caption{line clusters}
	\end{subfigure}
	\\
	\caption{Line cluster extraction from the \emph{desk-normal} sequence in the VECtor dataset. (a) An event frame generated by accumulating events, where red and blue dots represent events with opposite polarities. (b) The corresponding image. (c) Results of line segment detection. (d) Line cluster extraction by associating events near the line segments. If an event is within $5$ pixels of a line segment, draw a circle at the event's position.}
	\label{fig:line-vector}
\end{figure*}

We use the GC-RANSAC~\cite{barath2022graph} as the robust estimation framework. 
An angular reprojection~\cite{kneip2014opengv,lee2019closed} threshold of $0.2^\circ$ is used for inlier selection, applied consistently throughout both the main iterations and the local refinement stages. The spherical radius $r$ in neighborhood graph construction is set as $50$. 
The main iterations stop once there is a $0.99$ probability that at least one outlier-free set has been sampled, or when the maximum of $1000$ iterations is reached. In the local refinement stage, the maximum number of inner RANSAC iterations is set to $20$.

The results of line cluster extraction are shown in \cref{fig:line-vector}. Events are accumulated into frames, which has a resolution of $640 \times 480$. Line segments are then extracted from the undistorted frames, and events located within $5$ pixels of these line segments are considered as part of the line clusters. 

We have also conducted a comparison with two methods with known angular velocities \cite{gao2023fivepoint,gao2024npoint}. Using the same dataset and settings described in Table~2 of our paper, we obtained the comparison results of estimated linear velocities reported in \cref{tab:real-rebuttal}. 
We can see that the solvers in \cite{gao2023fivepoint,gao2024npoint} have better results than our proposed full-DoF solvers.
This is reasonable, as the comparison methods leverage IMU readings to solve the rotational parameters of the motion.

\begin{table}[H]
	\centering
	\setlength{\tabcolsep}{3.1pt} 
	\begin{tabular}{|l|c|c|c|c|}
		\hline
		Seq. Name  & \scenario{ICCV23} \cite{gao2023fivepoint} & \scenario{CVPR24} \cite{gao2024npoint} & \scenario{IncBat} & \scenario{CopBat} \\ \hline
		\emph{desk}     & $23.5$ & $21.8$ & $23.0$ & $25.1$ \\ \hline
		\emph{mountain} & $18.2$ & $16.4$ & $17.5$ & $18.7$ \\ \hline
		\emph{sofa}     & $19.7$ & $18.9$ & $21.1$ & $20.6$ \\ \hline
	\end{tabular}
	\vspace{-6pt}
	\caption{Experiment results of the comparison methods. We report the median errors of recovered linear velocities $\varepsilon_{\text{lin}}$ (unit: degree). This table supplements Table 2 of our paper.}
	\label{tab:real-rebuttal}
\end{table}

\subsection{Discussion}
Our method shows promising results, it also has a few limitations.
(1) The proposed method is best-suited for scenes that contain long straight lines.
(2) The performance of certain solvers heavily relies on the quality of the extracted normal flow. 
(3) For short time intervals, assuming constant angular and linear velocities is reasonable. However, during aggressive or non-uniform motions, sudden velocity changes may violate this assumption. A more appropriate approach would involve using a continuous-time model, such as a cubic B-spline model~\cite{mueggler2018continuous}, to model the $\omg(t)$ and $\v(t)$.

There are several potential strategies to reduce the dependence on line and normal flow extraction. First, it is possible to detect event clusters without identifying line segments, which makes it easier to eliminate the need for frame accumulation. When integrating our solver in a hypothesis-and-test framework (such as RANSAC), we can then verify whether the events within the same cluster originate from a line. Second, normal flow can be calculated directly without accumulating frames. The primary focus of our work lies on the relative pose solver, which remains independent of the line/cluster detection and normal flow extraction method. Our method can ultimately be combined with any front-end method.


\end{document}